\documentclass[twoside,11pt]{article}

\usepackage{arxiv-macros-new}

\begin{document}

\title{Universal Lower Bound for Causal Structure Learning with Interventions}
\author{Vibhor Porwal\thanks{Adobe Research,
    Bangalore. \href{mailto:vibhorporwal99@gmail.com}{\texttt{vibhorporwal99@gmail.com}}}
  \and Piyush Srivastava\thanks{Tata Institute of Fundamental Research,
    Mumbai. \href{mailto:piyush.srivastava@tifr.res.in}{\texttt{piyush.srivastava@tifr.res.in}}}
  \and Gaurav Sinha\thanks{Adobe Research,
    Bangalore. \href{mailto:gasinha@adobe.com}{\texttt{gasinha@adobe.com}}}}
\maketitle
\begin{NoHyper} {\let\thefootnote\relax \footnotetext{An earlier version of this
      paper was presented at AISTATS 2022 under the title ``Almost Optimal
      Universal Lower Bound for Learning Causal DAGs with Atomic
      Interventions''.}  }
\end{NoHyper}
\date{}

\begin{abstract}A well-studied challenge that arises in the structure learning problem of causal directed acyclic graphs (DAG) is that using observational data, one can only learn the graph up to a ``Markov equivalence class'' (MEC). The remaining undirected edges have to be oriented using interventions, which can be very expensive to perform in applications. Thus, the problem of minimizing the number of interventions needed to fully orient the MEC has received a lot of recent attention, and is also the focus of this work.
Our first result is a new universal lower bound on the number of single-node interventions that any algorithm (whether active or passive) would need to perform in order to orient a given MEC. Our second result shows that this bound is, in fact, within a factor of two of the size of the smallest set of single-node interventions that can orient the MEC. Our lower bound is provably better than previously known lower bounds. Further, using simulations on synthetic graphs and by giving examples of special graph families, we show that our bound is often significantly better. To prove our lower bound, we develop the notion of \Gcb{} (\gcb{}) orderings, which are topological orderings of DAGs without v-structures and satisfy certain special properties. We also use the techniques developed here to extend our results to the setting of multi-node interventions.
\end{abstract}

\newpage

\tableofcontents

\newpage

\section{Introduction}
Causal Bayesian Networks (CBN) provide a very convenient framework for modeling
causal relationships between a collection of random
variables~\citep{Pearl2009}.  A CBN is fully specified by (a) a directed acyclic
graph (DAG), whose nodes model random variables of interest, and whose edges
depict immediate causal relationships between the nodes, and (b) a conditional
probability distribution (CPD) of each variable given its parent variables (in
the DAG) such that the joint distribution of all variables factorizes as a
product of these conditionals.  The generality of the framework has led to CBN
becoming a popular tool for the modeling of causal relationships in a 
variety of fields, with health science~\citep{Shen2020}, molecular cell
biology~\citep{friedman_inferring_2004}, and computational advertising
\citep{Bottou2013} being a few examples.

It is well known that the underlying DAG of a CBN is not uniquely determined by
the joint distribution of its nodes. In fact, the joint distribution only
determines the DAG up to its Markov Equivalence Class (MEC), which is
represented as a partially directed graph with well-defined combinatorial
properties~\citep{verma_equivalence_1990,chickering_transformational_1995,meek_causal_1995,andersson_characterization_1997}.
Information about which nodes are \emph{adjacent} is encoded in the MEC, but the
\emph{direction} of several edges remains undetermined. Thus, learning
algorithms based only on the observed joint distribution \citep{Glymour2019}
cannot direct these remaining edges. As a result, algorithms which use
additional \emph{interventional} distributions were developed
(\cite{squires2020} and references therein). In addition to the joint
distribution, these algorithms also assume access to interventional
distributions generated as a result of randomizing some \emph{target} vertices
in the original CBN (a process called \emph{intervention}) and thereby breaking
their dependence on any of their ancestors. A natural and well-motivated
\citep{eberhardt_number_2012} question, therefore, is to find the minimum number
of interventions required to fully resolve the orientations of the undirected
edges in the MEC.

Interventions, especially on a large set of nodes, however, can be expensive to
perform \citep{kocaoglu_cost-optimal_2017}. In this respect, the setting of
\emph{atomic interventions}, where each intervention is on a single node, is
already very interesting and finding the smallest number of atomic interventions
that can orient the MEC is well studied \citep{squires2020}.
A long line of work, including those cited above, has considered in various
settings the problem of designing methods for finding the smallest set of atomic
interventions that would fully orient all edges of a given MEC.  An important
distinction between such methods is whether they are \emph{active}~\citep{He2008},
i.e., where the directions obtained via the current intervention are available
before one decides which further interventions to perform; or \emph{passive},
where all the interventions to be performed have to be specified beforehand.
Methods can also differ in whether or not randomness is used in selecting the
targets of the interventions.  An important question, therefore, is to
understand how many interventions must be performed by any given method to
fully orient an MEC.

\paragraph{Universal Lower Bounds} While several works have reported lower bounds
(on minimum number of atomic interventions required to orient an MEC) in different
settings, a very satisfying solution concept for such lower bounds, called
\emph{universal lower bounds}, was proposed by~\cite{squires2020}.  A universal
lower bound of $L$ atomic interventions for orienting a given MEC means that if a
set of atomic interventions is of size less than $L$, then for \emph{every}
ground-truth DAG $D$ in the MEC, the set $S$ will fail to fully orient the MEC.
Thus, a universal lower bound has two universality properties. First, the value of
a universal lower bound depends only upon the MEC, and applies to \emph{every} DAG
in the MEC.  Second, the lower bound applies to \emph{every} set of interventions
that would fully orient the MEC, without regards to the method by which the
intervention set was produced.

In this work, we address the problem of obtaining tight universal lower bounds.
The goal is to find a universal lower bound such that for any DAG $D$ in the MEC,
the smallest set of atomic interventions that can orient the MEC into $D$ has size
bounded above by a constant factor of the universal lower bound. Similar to
\cite{squires2020}, we work in the setting of \emph{causally sufficient} models,
i.e. there are no hidden confounders, selection bias or feedback. To the best of
our knowledge, this is the first work that addresses the problem of tight (up to
a constant factor) universal lower bounds. We note that the best known universal
lower bounds \citep{squires2020} so far are not tight and provide concrete examples
of graph families that illustrate this in \Cref{subsubsec:examples}.

\subsection{Our Contributions and Organization of the Paper}
We prove a new universal lower bound on the size of any set of atomic
interventions that can orient a given MEC, improving upon previous
work~\citep{squires2020}.  We further prove that our lower bound is optimal
within a factor of 2 in the class of universal lower bounds: we show that for
any DAG $D$ in the MEC, there is a set of atomic interventions of size at most
twice our lower bound, that would fully orient the MEC if the unknown
ground-truth DAG were $D$. 

We also compare our new lower bound with the one obtained previously by
\cite{squires2020}.  We prove analytically that our lower bound is at least as
good as the one given by~\cite{squires2020}.  We further give examples of graph
classes where our bound is significantly better (in fact, it is apparent from
our proof that the graphs in which the two lower bounds are close must have very
special properties).  We then supplement these theoretical findings with
simulation results comparing our lower bound with the ``true'' optimal answer
and with the lower bound in previous work.

Further, using the techniques developed in our work, we explore how tight our
lower bound remains in the setting of multi-node interventions. See \Cref{sec:multi-node-interventions} for more details.

Our lower bound is based on elementary combinatorial arguments drawing upon the
theory of chordal graphs, and centers around a notion of certain special
topological orderings of DAGs without v-structures, which we call \emph{\Gcb{}
  (\gcb{}) orderings} (\cref{def:nice-ordering}).  This is in contrast to the
earlier work of \cite{squires2020}, where they had to develop sophisticated
notions of directed clique trees and residuals in order to prove their lower
bound.  We expect that the notion of \gcb{} orderings may also be of interest in
the design of optimal intervention sets.

Many of the proofs are deferred to the appendix.  In particular,
Section~\ref{sec:some-folkl-results} of the appendix gives, for the sake of
completeness, proofs of many folklore observations concerning the notion of
interventional Markov equivalence.

\subsection{Related Work}

The theoretical underpinning for many works dealing with the use of
  interventions for orienting an MEC can be said to be the notion of
  ``interventional'' Markov equivalence~\citep{HB12}, which, roughly speaking,
  says that given a collection $\mathcal{I}$ of sets of targets for
  interventions, two DAGs $D_1$ and $D_2$ are \emph{$\mathcal{I}$-Markov
    equivalent} if and only if for all $S \in \mathcal{I}$, the DAGs obtained by
  removing from $D_1$ and $D_2$ the incoming edges of all vertices in $S$ are in
  the same MEC~\citep[Theorem 10]{HB12}.  Thus, interventions have the
  capability of distinguishing between DAGs in the same Markov Equivalence
  class, and in particular, ``interventional'' Markov equivalence classes can be
  finer than MECs~\citep[see also \cref{thm:hb-i-essential} below]{HB12}.

  As described above, the problem of learning the orientations of a CBN using
  interventions has been studied in a wide variety of settings.  Lower bounds
  and algorithms for the problem have been obtained in the setting of
  interventions of arbitrary sizes and with various cost
  models~\citep{eberhardt_almost_2012,shanmugamKDV15,kocaoglu_cost-optimal_2017},
  in the setting when the underlying model is allowed to contain feedback loops
  (and is therefore not a CBN in the usual
  sense)~\citep{hyttinen_experiment_2013, hyttinen_discovering_2013}, in
  settings where hidden variables are
  present~\citep{addanki_efficient_2020,addanki_intervention_2020},
  and in interventional ``sample efficiency'' settings~\citep{AgrawalSYSU19,greenewald_sample_2019}.
  The related notion of orienting the maximum possible number of edges given a
  fixed budget on the number or cost of interventions has also been
  studied~\citep{hauser_two_2014,ghassami2018budgeted,AhmadiTeshniziS20}.
  However, to the best of our knowledge, the work of \cite{squires2020} was the
  first to isolate the notion of a universal lower bound, and prove a lower bound
  in that setting.

\section{Preliminaries}
\label{sec:preliminaries}
\paragraph{Graphs} A \emph{partially directed graph} (or just \emph{graph})
$G = (V, E)$ consists of a set $V$ of \emph{nodes} or \emph{vertices} and a set
$E$ of \emph{adjacencies}.  Each adjacency in $E$ is of the form $a \undir{} b$
or $a \rightarrow b$, where $a, b \in V$ are \emph{distinct} vertices, with the
condition that for any $a, b \in V$, at most one of $a \undir b$, $a \dir b$ and
$b \rdir a$ is present in $E$.\footnote{$a \undir b$ and $b \undir a$ are
  treated as equal.}
If there is an adjacency in $E$ containing both $a$ and
$b$, then we say that $a$ and $b$ are \emph{adjacent} in $G$, or that there is
an \emph{edge} between $a$ and $b$ in $G$.  If $a \undir b \in E$, then we say
that the edge between $a$ and $b$ in $G$ is undirected, while if
$ a \dir b \in E$ then we say that the edge between $a$ and $b$ is
\emph{directed} in $G$.  $G$ is said to be \emph{undirected} if all its
adjacencies are undirected, and \emph{directed} if all its adjacencies are
directed.  Given a directed graph $G$, and a vertex $v$ in $G$, we denote by
$\pa[G]{v}$ the set of nodes $u$ in $G$ such that $u \dir v$ is present in $G$.
A vertex $v$ in $G$ is said to be a \emph{child} of $u$ if $u \in \pa[G]{v}$.
An \emph{induced subgraph} of $G$ is a graph whose vertices are some subset $S$
of $V$, and whose adjacencies $E[S]$ are all those adjacencies in $E$ both of
whose elements are in $S$. This induced subgraph of $G$ is denoted as $G[S]$.
The \emph{skeleton} of $G$, denoted $\skel{G}$, is an undirected graph with
nodes $V$ and adjacencies $a \undir b$ whenever $a$, $b$ are adjacent in $G$.
 
A \emph{cycle} in a graph $G$ is a sequence of vertices
$v_1, v_2, v_3, \dots, v_{n + 1} = v_1$ (with $n \geq 3$) such that for each
$1 \leq i \leq n$, either $v_{i} \dir v_{i + 1}$ or $v_{i} \undir v_{i+1}$ is
present in $E$.  The \emph{length} of the cycle is $n$, and the cycle is said to
be \emph{simple} if $v_1, v_2, \dots, v_n$ are distinct.  The cycle is said to
have a \emph{chord} if two non-consecutive vertices in the cycle are adjacent in
$G$, i.e., if there exist $1 \leq i < j \leq n$ such that
$j - i \neq \pm 1\; (\text{mod } n)$ and such that $v_i$ and $v_j$ are adjacent
in $G$.  The cycle is said to be \emph{directed} if for some $1 \leq i \leq n$,
$v_{i} \dir v_{i + 1}$ is present in $G$. A graph $G$ is said to be a
\emph{chain graph} if it has no directed cycles. The \emph{chain components} of
a chain graph $G$ are the connected components left after removing all the
directed edges from $G$.  A \emph{directed acyclic graph} or DAG is a directed
graph without directed cycles.  Note that both DAGs and undirected graphs are
chain graphs.  An undirected graph $G$ is said to be \emph{chordal} if any
simple cycle in $G$ of length at least 4 has a chord.

A \emph{clique} $C$ in a graph $G = (V, E)$ is a subset of nodes of $G$ such
that any two distinct $u$ and $v$ in $C$ are adjacent in $G$. The clique $C$ is
\emph{maximal} if for all $v \in V \setminus C$, the set $C \cup \inb{v}$ is not
a clique.

A \emph{perfect elimination ordering} (PEO), $\sigma = (v_1, \dots, v_n)$ of a
graph $G$ is an ordering of the nodes of $G$ such that $\forall i \in [n]$,
$ne_G(v_i) \cap \{v_1, \dots, v_{i-1}\}$ is a clique in $G$, where $ne_G(v_i)$
is the set of nodes adjacent to $v_i$.\footnote{Our definition of a PEO uses the
  same ordering convention as \cite{hauser_two_2014}.} A graph is chordal if and
only if it has a perfect elimination ordering~\citep{blair_introduction_1993}. A
\emph{topological ordering}, $\sigma$ of a DAG $D$ is an ordering of the nodes
of $D$ such that $\sigma(a) < \sigma(b)$ whenever $a \in \pa[D]{b}$, where
$\sigma(u)$ denotes the index of $u$ in $\sigma$. We say that $D$ is oriented
according to an ordering $\sigma$ to mean that $D$ has a topological ordering
$\sigma$.

A \emph{v-structure} in a graph $G$ is an induced subgraph of the form
$b \rightarrow a \leftarrow c$ (v-structures are also known as
\emph{unshielded colliders}). It follows easily from the definitions that by
orienting the edges of a chordal graph according to a perfect elimination
ordering, we get a DAG without v-structures, and that the skeleton of a DAG
without v-structures is chordal (see Proposition 1 of \cite{hauser_two_2014}).
In fact, any topological ordering of a DAG $D$ without v-structures is a perfect
elimination ordering of $\skel{D}$.

\paragraph{Interventions} An \emph{intervention} $I$ on a partially directed graph
$G$ is specified as a subset of \emph{target} vertices of $G$.
Operationally, an intervention at $I$ is interpreted as completely randomizing
the distributions of the random variables corresponding to the vertices in $I$.
We work here in the ``infinite sample'' setting, where, under standard
assumptions, performing the intervention $I$ reveals at least the directions
of all edges between vertices in $I$ and $V \setminus I$ (see
\cref{thm:hb-i-essential} below for a more formal statement of the extent to
which a set of interventions orients the edges of a partially directed graph).
An \emph{intervention set} is a set of interventions.  In this paper, we make the
standard assumption that the ``empty'' intervention, in which no vertices are
intervened upon, is \emph{always} included in any intervention set we consider:
this corresponds to assuming that information from purely observational data is
always available (see, e.g., the discussion surrounding Definition 6 of \cite{HB12}).
The \emph{size} of an {intervention set} $\mathcal{I}$ is the number of
interventions in $\mathcal{I}$, not counting the empty intervention.

$\mathcal{I}$ is a set of \emph{atomic interventions} if $\abs{I} = 1$ for all
non-empty $I \in \mathcal{I}$ (an intervention $I$ is said to \emph{non-atomic} if $\abs{I} > 1$).  With a slight abuse of notation, we denote
a set of atomic interventions
$\mathcal{I} = \inb{\emptyset, \inb{v_1}, \dots, \inb{v_k}}$ as just the set
$I = \inb{v_1, \dots, v_k}$ when it is clear from the context that we are
talking about a set of atomic interventions.
\begin{figure*}[t]
    \centering
    \renewcommand{\thesubfigure}{\roman{subfigure}}
    \begin{subfigure}[b]{0.24\textwidth}
      \centering
      \begin{tikzpicture}
        \node (a){$a$};
        \node[right of = a](b){$b$};
        \node[above left of = b](c){$c$};
        \draw[dir] (a) -- (b);
        \draw[dir] (c) -- (a);
      \end{tikzpicture}
      \caption{}
      \label{fig:B}
    \end{subfigure}
    \begin{subfigure}[b]{0.24\textwidth}
      \centering
      \begin{tikzpicture}
        \node (a){$a$};
        \node[above right of = a](c){$c$};
        \node[right of = a](b){$b$};
        \draw[dir] (a) -- (b);
        \draw[dir] (c) -- (b);
      \end{tikzpicture}
      \caption{}
      \label{fig:A}
    \end{subfigure}
    \begin{subfigure}[b]{0.24\textwidth}
      \centering
      \begin{tikzpicture}
        \node (a){$a$};
        \node[above right of = a](c){$c$};
        \node[below right of = c](b){$b$};
        \draw[dir] (a) -- (b);
        \draw[dir] (c) -- (b);
        \draw[dir] (a) -- (c);
      \end{tikzpicture}
      \caption{}
      \label{fig:C}
    \end{subfigure}
    \begin{subfigure}[b]{0.24\textwidth}
      \centering
      \begin{tikzpicture}
        \node (a){$a$};
        \node[left of = a](c1){$c_1$};
        \node[right of = a](c2){$c_2$};
        \node[below of = a](b){$b$};
        \draw[dir] (a) -- (b);
        \draw[dir] (c1) -- (b);
        \draw[dir] (c2) -- (b);
        \draw (c1) -- (a) -- (c2);
      \end{tikzpicture}
      \caption{}
      \label{fig:D}
    \end{subfigure}
    \vspace{\baselineskip}
    \caption{Strong
      Protection~\citep{andersson_characterization_1997,HB12}
      \label{fig:strong-protection}}
  \end{figure*}
Given an intervention set $\mathcal{I}$ and a DAG $D$, we denote, following
\cite{HB12}, by $\mathcal{E}_{\mathcal{I}}(D)$ the partially directed graph
representing the set of all DAGs that are $\mathcal{I}$-Markov equivalent to
$D$. $\mathcal{E}_{\mathcal{I}}(D)$ is also known as the $\mathcal{I}$-essential
graph of $D$ (see Fig. 2 of \cite{hauser_two_2014} for an example).
For a formal definition of $\mathcal{I}$-Markov equivalence, we
refer to Definitions~7 and 9 of \cite{HB12}; we use instead the following
equivalent characterization developed in the same paper.
\begin{restatable}[\textbf{Characterization of $\mathcal{I}$-essential graphs,
Definition 14 and Theorem 18 of \cite{HB12}}]{theorem}{hbcharthm} \label{thm:hb-i-essential}
    Let $D$ be a DAG and $\mathcal{I}$ an intervention set containing the empty set. A graph $H$ is an
    $\mathcal{I}$-essential graph of $D$ if and only if $H$ has the same
    skeleton as $D$, all directed edges of $H$ are directed in the same
    direction as in $D$, all v-structures of $D$ are directed in $H$, and
  \begin{enumerate}
  \item \label{item:chain-chordal} $H$ is a chain graph with chordal chain components.
  \item \label{item:directed-by-Meek-rule-1} For any three vertices $a, b, c$ of $H$,
  the subgraph of $H$ induced by $a$, $b$ and $c$ is not $a \rightarrow b - c$.
  \item \label{item:directed-by-intervention} If $a \rightarrow b$ in $D$ (so
    that $a$, $b$ are adjacent in $H$) and there is an intervention
    $J \in \mathcal{I}$ such that $\abs{J \cap \inb{a, b}} = 1$, then
    $a \rightarrow b$ is directed in $H$.
  \item \label{item:strong-i-protection} Every directed edge $a \rightarrow b$
    in $H$ is strongly $\mathcal{I}$-protected.  An edge $a \rightarrow b$ in
    $H$ is said to be strongly $\mathcal{I}$-protected if either (a) there is an
    intervention $J \in \mathcal{I}$ such that $\abs{J \cap \inb{a, b}} = 1$, or (b) at
    least one of the four graphs in \Cref{fig:strong-protection} appears as an
    induced subgraph of $H$, and $a \rightarrow b$ appears in that induced
    subgraph in the configuration indicated in the figure.
  \end{enumerate}
\end{restatable}

\section{Universal Lower Bound}

In this section, we establish our main technical result
(\Cref{lemma:good-topological-ordering}). Our new lower bound
(\cref{theorem:main-theorem,theorem:main-theorem-arbitrary-DAGs}) then follows
easily from this combinatorial result, without having to resort to the
sophisticated machinery of residuals and directed clique trees developed in
previous work \citep{squires2020}.

We begin with a definition that isolates two important properties of certain
topological orderings of DAGs without v-structures.  Given a DAG $D$ without
v-structures, and a maximal clique $C$ of \skel{D}, we denote by $\sink[D]{C}$
any vertex in $D$ such that $C = \pa[D]{\sink[D]{C}} \cup \inb{\sink[D]{C}}$.
The fact that \sink[D]{C} is uniquely defined, and that
$\sink[D]{C_1} \neq \sink[D]{C_2}$ when $C_1$ and $C_2$ are distinct maximal
cliques of \skel{D} is guaranteed by the following observation. (The standard proof of
this is deferred to Section~\ref{sec:supp:proof-sink-nodes}.)
\begin{observation} \label{lemma:sink-nodes} Let $D$ be a DAG without
  v-structures.  Then, for every maximal clique $C$ of $\skel{D}$, there is a
  unique vertex $v$ of $D$, denoted $\sink[D]{C}$, such that
  $C = \pa[D]{v} \cup \inb{v}$.  Further, for any two distinct maximal cliques
  $C_1$ and $C_2$ in \skel{D}, we have $\sink[D]{C_1} \neq \sink[D]{C_2}$.
\end{observation}
We refer to each vertex $v$ of $D$ that is equal to $\sink[D]{C}$ for some
maximal clique of \skel{D} as a \emph{\sinkv{}} vertex of the DAG $D$. Note also
that $\sink[D]{C}$ is also the unique node with out-degree $0$ in the induced
subgraph $D[C]$.

\begin{definition}[\textbf{\GCB{} (\gcb{}) ordering}]\label{def:nice-ordering} Let
  $\sigma$ be a topological ordering of a DAG $D$ without v-structures. Let
  $s_1, s_2, s_3, \dots, s_r$ be the \sinkv{} vertices of $D$ indexed so that
  $\sigma(s_i) < \sigma(s_j)$ when $i < j$. (Here $r$ is the number of maximal
  cliques in \skel{D}.) Then, $\sigma$ is said to be a \emph{\Gcb{} (\gcb{}) ordering}
  of $D$ if it satisfies the following two properties:
  \begin{enumerate}
  \item \textbf{P1: Clique block property} Define $L_1(\sigma)$ to be the set of
    nodes $u$ which occur before or at the same position as $s_1$ in $\sigma$
    i.e., $\sigma(u) \leq \sigma(s_1)$. Similarly, for $2 \leq i \leq r$, define
    $L_i(\sigma)$ to be the set of nodes which occur in $\sigma$ before or at
    the same position as $s_i$, but strictly after $s_{i-1}$ (i.e.,
    $\sigma(s_{i-1}) < \sigma(u) \leq \sigma(s_i)$).
Then, for each $1\leq i \leq r$ the subgraph induced by $L_i(\sigma)$ in
    $\skel{D}$ is a (not necessarily maximal) clique.
  \item \textbf{P2: Shared parents property} If vertices $a$ and $b$ in $D$ are
    consecutive in $\sigma$ (i.e., $\sigma(b) = \sigma(a) + 1$), and also lie in
    the same $L_i(\sigma)$ for some $1 \leq i \leq r$, then all parents of $a$
    are also parents of $b$ in $D$.
  \end{enumerate}
\end{definition}
\begin{figure*}[t]
  \centering
  \begin{subfigure}[b]{0.39\textwidth}
    \centering
    \begin{tikzpicture}
      \node (a){$a$};
\node (b)[above of=a]{$\ubar{b}$};
\node (c)[right of=b]{$c$};
\node (d)[above of=c]{$\ubar{e}$};
\node (e)[right of=c]{$d$};
\node (f)[below of=c]{$\ubar{f}$};
\draw[dir] (a) -- (b);
\draw[dir] (b) -- (c);
      \draw[dir] (b) -- (d);
\draw[dir] (c) -- (d);
      \draw[dir] (c) -- (e);
      \draw[dir] (c) -- (f);
\draw[dir] (e) -- (f);
    \end{tikzpicture}\\
    A DAG $D$ without v-structures
  \end{subfigure}
  \begin{subfigure}[b]{0.6\textwidth}
    \centering
    \begin{align*}
      \sigma'
      &\defeq \underbrace{a\,\ubar{b}}_{L_1(\sigma')}\quad
        \underbrace{c\,d\,\ubar{e}}_{L_2(\sigma')}\quad
        \underbrace{\ubar{f}}_{L_3(\sigma')}\\
      \sigma
      &\defeq \underbrace{a\,\ubar{b}}_{L_1(\sigma)}\quad
        \underbrace{c\,d\,\ubar{f}}_{L_2(\sigma)}\quad
        \underbrace{\ubar{e}}_{L_3(\sigma)}\\
      \tau
      &\defeq \underbrace{a\,\ubar{b}}_{L_1(\tau)}\quad
        \underbrace{c\,\ubar{e}}_{L_2(\tau)}\quad
        \underbrace{d\,\ubar{f}}_{L_3(\tau)}
    \end{align*}
  \end{subfigure}
  \vspace{0.3\in}

  \caption{\GCB{} Topological Orderings\label{fig:gcb-ordering}: $\tau$
    Satisfies both P1 and P2; $\sigma$ Satisfies only P1}
\end{figure*}
We illustrate the definition with an example in \Cref{fig:gcb-ordering}.  In the
figure, vertices $\ubar{b}$, $\ubar{e}$, and $\ubar{f}$ are the \sinkv{}
vertices of $D$, and are highlighted with an underbar.  The orderings $\sigma'$,
$\sigma$ and $\tau$ in the figure are valid topological orderings of $D$.
However, $\sigma'$ does not satisfy P1 of \cref{def:nice-ordering} (since
$L_2(\sigma')$ is not a clique), while $\sigma$ satisfies P1 of
\cref{def:nice-ordering}, but does not satisfy P2, because $c, d$ in
$L_2(\sigma)$ are consecutive in $\sigma$, but $b$ is a parent only of $c$ and
not of $d$.  Finally, $\tau$ satisfies both P1 and P2 and hence is a \gcb{}
ordering.

Our main technical result is that for any DAG $D$ that has no v-structures,
there exists a \gcb{} ordering $\sigma$ of $D$, and the new lower bound is an
easy corollary of this result.  Further, the proof of this result uses only
standard notions from the theory of chordal graphs.
\begin{theorem}\label{lemma:good-topological-ordering}
  If $D$ is a DAG without v-structures, then $D$ has a \gcb{} ordering.
\end{theorem}

Towards the proof of this theorem, we note first that the existence of a
topological ordering $\sigma$ satisfying just P1 can be established using ideas
from the analysis of, e.g., the ``maximum cardinality search'' algorithm for
chordal graphs (\citet{tarjan_simple_1984}, see also Corollary 2 of
\cite{wienobst_polynomial-time_2020}).  We state this here as a lemma, and
provide the proof in Section~\ref{sec:supp:proof-p1-ordering}.
\begin{lemma}\label{lem:P1-ordering}
  If $D$ is a DAG without v-structures, then $D$ has a topological ordering
  $\sigma$ satisfying P1 of \cref{def:nice-ordering}.
\end{lemma}
We now prove the theorem.
\begin{proof}[Proof of \cref{lemma:good-topological-ordering}]
  Let $\mathcal{O}$ be the set of topological orderings of $D$ which satisfy P1
  of \cref{def:nice-ordering}.  By \cref{lem:P1-ordering}, $\mathcal{O}$ is
  non-empty.  If there is a $\sigma \in \mathcal{O}$ which also satisfies P2 of
  \cref{def:nice-ordering}, then we are done.

  We now proceed to show by contradiction that such a $\sigma$ must indeed
  exist.  So, suppose for the sake of contradiction that for each $\sigma$ in
  $\mathcal{O}$, P2 is violated. Then, for each $\sigma \in \mathcal{O}$, there
  exist vertices $a, b$ and an index $i$ such that $a, b \in L_i(\sigma)$,
  $\sigma(b) = \sigma(a) + 1$, and there exists a parent of $a$ in $D$ that is
  not a parent of $b$.  For any given $\sigma \in \mathcal{O}$, we choose $a, b$
  as above so that $\sigma(a)$ is as small as possible.  With such a choice of
  $a$ for each $\sigma \in \mathcal{O}$, we then define a function
  $f: \mathcal{O} \rightarrow [n-1]$ by defining $f(\sigma) = \sigma(a)$.  Note
  that by the assumption that P2 is violated by each $\sigma$ in $\mathcal{O}$,
  $f$ is defined for each $\sigma$ in $\mathcal{O}$.  But then, since
  $\mathcal{O}$ is a finite set, there must be some $\sigma \in \mathcal{O}$ for
  which $f(\sigma)$ attains its maximum value.  We obtain a contradiction by
  exhibiting another $\tau \in \mathcal{O}$ for which $f(\tau)$ is
  \emph{strictly} larger than $f(\sigma)$.  We first describe the construction
  of $\tau$ from $\sigma$, and then prove that $\tau$ so constructed is in
  $\mathcal{O}$ and has $f(\tau) > f(\sigma)$.

  \textit{Construction.} Given $\sigma \in \mathcal{O}$, let
  $f(\sigma) = j \in [n-1]$. Let $\sigma(a) = j, \sigma(b) = j+1$ and suppose
  that $a, b \in L_i(\sigma)$.  By the definition of $f$, there is a parent of
  $a$ that is not a parent of $b$.  Define $C^a$ to be the set
  $\inb{a} \cup \pa[D]{a}$.  Since $D$ has no v-structures, $C^a$ is a clique in
  \skel{D}.  Define \(S^a \defeq \{z | C^a \subseteq \pa[D]{z}\}.\) Let $Y_a$ be
  the set of nodes occurring after $a$ in $\sigma$ that are not in $S^a$ (note
  that $b \in Y_a$, and, in general, $z \in Y_a$ if and only if
  $\sigma(z) > \sigma(a)$ and there is an $x \in C^a$ that is not adjacent to
  $z$).
  
  We now note the following easy to verify properties of the sets $S^a$ and
  $Y_a$ (the proof is provided in
  Section~\ref{sec:supp:proof-prop-intermediate}).
  \begin{proposition}\label{prop:intermediate}
    \begin{enumerate}
    \item \label{item:1} If $y \in Y_a$ and $z$ is
      a child of $y$ in $D$, then $z \in Y_a$.\item \label{item:2} Suppose that $x \in \inb{a} \cup S^a$ is not a \sinkv{}
      node of $D$.  Then there exists $y \in S^a$ such that $x \in \pa[D]{y}$
      and such that $y$ is a \sinkv{} node in $D$. In particular, $S^a$ is non-empty.
    \item \label{item:3} Suppose that $x$ is a \sinkv{} node in the induced DAG
      $H \defeq D[S^a]$.  Then $x$ is also a \sinkv{} node in $D$.
    \end{enumerate}
  \end{proposition}
  
  The ordering $\tau$ is now defined as follows: the first $j$ nodes in $\tau$
  are the same as $\sigma$.  After this, the nodes of $S^a$ appear according to
  some topological ordering $\gamma$ of the induced DAG $H \defeq D[S^a]$ that
  satisfies P1 of \cref{def:nice-ordering} in $H$ (such a $\gamma$ exists
  because of \cref{lem:P1-ordering} applied to the induced DAG $H$, which also
  cannot have any v-structures). Finally, the nodes of $Y_a$ appear according to
  their ordering in $\sigma$.  

  \textit{Proof that $\tau \in \mathcal{O}$ and $f(\tau) > f(\sigma)$.}  Note
  first that $\tau$ is a topological ordering of $D$: if not, then there
  must exist $u \in Y_a$ and $v \in S^a$ such that the edge $u \rightarrow v$ is
  present in $D$, but this cannot happen by \cref{item:1} of
  \cref{prop:intermediate} above.

  To show that $\tau \in \mathcal{O}$ (i.e., that $\tau$ satisfies P1), the
  following notation will be useful.  For each \sinkv{} vertex $s$ in $D$,
  denote by $\lambda(s)$ the unique $L_\alpha(\sigma)$ such that
  $s \in L_\alpha(\sigma)$.  Similarly, denote by $\mu(s)$ the unique
  $L_\beta(\tau)$ such that $s \in L_\beta(\tau)$.  Since
  $\sigma \in \mathcal{O}$, we already know that $\lambda(s)$ is a clique for
  each \sinkv{} node $s$ of $D$.  In order to show that $\tau \in \mathcal{O}$,
  all we need to show is that $\mu(s)$ is also a clique for each \sinkv{} node
  $s$ of $D$.

  Let $J$ be the set of \sinkv{} vertices of $D$ present in $S^a$. By
  \cref{item:2} of \cref{prop:intermediate}, the last vertex in $\gamma$ must be
  an element of $J$.  Note also that $s_i \not\in J$ since $b \in Y_a$ and the
  edge $b \rightarrow s_i$ in $D$ together imply that $s_i \in Y_a$ by
  \cref{item:1} of \cref{prop:intermediate}.  We also observe that
  $\lambda(s) \subseteq S^a$, for all $s \in J$.  For if there exists
  $u \in \lambda(s) \cap Y_a$ then the edge $u \rightarrow s$ in $D$ implies
  that $s \in Y_a$, contradicting that $s \in S^a$.

  From the construction of $\tau$, we already have $\mu(s_j) = \lambda(s_j)$ for
  all \sinkv{} vertices $s_j$ that precede $s_i$ in $\sigma$.  From the fact
  that all vertices in $S^a$ precede $Y_a$ in the ordering $\tau$, and from the
  observations above that (i) the sink node $s_i \in Y_a$, and (ii)
  $\lambda(s) \subseteq S^a$, for all sink nodes $s \in S^a$, we also get that
  for any \sinkv{} node $s$ of $D$ such that $s \in Y_a$,
  $\mu(s) \subseteq \lambda(s)$.  Thus, when $s$ is a \sinkv{} node of $D$ that
  is not in $J \subset S^a$, we have that $\mu(s)$ is a clique in \skel{D},
  since $\mu(s) \subseteq \lambda(s)$, and $\lambda(s)$ is a clique in \skel{D}.
  It remains to show that $\mu(s)$ is a clique when $s \in J$.

  Let $t_1, t_2, \dots, t_k$ be the \sinkv{} nodes of $H = D[S^a]$, arranged in
  increasing order by $\gamma$.  Since $\gamma$ satisfies P1 in $H$, each
  $L_i(\gamma)$, $1 \leq i \leq k$, is a clique in $H$ (and thus also in $D$).
  Now consider a \sinkv{} node $s \in J \subseteq S^a$. Since the $t_i$ are
  \sinkv{} nodes of $D$ (from \cref{item:3} of \cref{prop:intermediate}), it
  follows that $\mu(s) \subseteq L_i(\gamma)$ (if $s \in L_i(\gamma)$ for
  $i \geq 2$) or $\mu(s) \subseteq L_1(\gamma) \cup C^a$ (if
  $s \in L_1(\gamma)$).  In the former case, $\mu(s)$ is automatically a clique,
  since $L_i(\gamma)$ is a clique in $D$.  In the latter case also $\mu(s)$ is a
  clique since $L_1(\gamma) \subseteq S^a$, so that $L_1(\gamma) \cup C^a$ is a
  clique in $D$ (since (i) $L_1(\gamma)$ and $C^a$ are cliques, and (ii) by
  definition of $S^a$, every node of $S^a$ is adjacent to every node in $C^a$).

  Thus, we get that $\tau$ also satisfies P1, so that $\tau \in \mathcal{O}$.
  Consider the node $b' \defeq \tau(j + 1)$ next to $a$ in $\tau$. Since $S^a$
  is non-empty, the construction of $\tau$ implies $b' \in S^a$, so that $b'$ is
  adjacent to all parents of $a$.  Since $\sigma$ and $\tau$ agree on the
  ordering of all vertices up to $a$, we thus have
  $f(\tau) \geq \tau(b') = j + 1 > j = f(\sigma)$.  This gives the desired
  contradiction to $\sigma$ being chosen as a maximum of $f$. Thus, there must
  exist some ordering in $\mathcal{O}$ which satisfies P2.
\end{proof}

At the heart of our lower bound proof is the following important property of
CBSP orderings.  Our lower bound for atomic interventions on DAGs without
v-structures, \cref{theorem:main-theorem}, immediately follows once we isolate
this property.

\begin{lemma}
  \label{lem:technical-lower-bound}
  Let $D$ be a DAG without v-structures, and let $\sigma$ be a CBSP ordering of
  nodes in $D$ (at least one such $\sigma$ exists by
  \Cref{lemma:good-topological-ordering}).  Let $a$ and $b$ be any nodes in $D$
  that are consecutive in $\sigma$ (i.e., $\sigma(b) = \sigma(a) + 1$) which lie
  in the same clique block $L_i(\sigma)$ of $\sigma$ (in particular the edge
  $a \rightarrow b$ is present in $D$).  Let $\mathcal{I}$ be an intervention
  set containing the empty intervention.  Then, the edge $a \rightarrow b$ is
  directed in $\mathcal{E}_{\mathcal{I}}(D)$ if and only if there is an intervention
  $I \in \mathcal{I}$ such that $\abs{I \cap \inb{a, b}} = 1$.
\end{lemma}

\begin{proof}
  In case there exists an $I \in {\mathcal{I}}$ such that
  $\abs{I \cap \inb{a, b}} = 1$, the edge $a \rightarrow b$ is directed in
  $\mathcal{E}_{\mathcal{I}}(D)$ by \cref{item:directed-by-intervention} of
  \Cref{thm:hb-i-essential}.

  We now show that if $\abs{I \cap \inb{a, b}} \neq 1$ for every
  $I \in \mathcal{I}$, then the edge $a \rightarrow b$ is not directed in
  $\mathcal{E}_{\mathcal{I}}(D)$.  Suppose, for the sake of contradiction, that
  $a \rightarrow b$ is directed in $\mathcal{E}_{\mathcal{I}}(D)$.  Then, by
  \cref{item:strong-i-protection} of \cref{thm:hb-i-essential},
  $a \rightarrow b$ must be strongly $\mathcal{I}$-protected in
  $\mathcal{E}_I(D)$.  Since $\abs{I \cap \inb{a, b}} \neq 1$ for every
  $I \in \mathcal{I}$, one of the graphs in \Cref{fig:strong-protection} must
  therefore appear as an induced subgraph of $\mathcal{E}_{\mathcal{I}}(D)$.  We
  now show that none of these subgraphs can appear as an induced subgraph of
  $\mathcal{E}_{\mathcal{I}}(D)$.

  First, subgraphs (ii) and (iv) cannot be induced subgraphs of
  $\mathcal{E}_{\mathcal{I}}(D)$ since they have a v-structure at $b$ while $D$
  (and therefore also $\mathcal{E}_{\mathcal{I}}(D)$) has no v-structures.  For
  subgraph (iii) to appear as an induced subgraph, the vertex $c$ must lie
  between $a$ and $b$ in any topological ordering of $D$, which contradicts the
  fact that $a$ and $b$ are consecutive in the topological ordering $\sigma$.
  For subgraph (i) to appear, we must have a parent $c$ of $a$ that is not
  adjacent to $b$. However, since $\sigma$ is a \gcb{} ordering, it satisfies
  property P2 of \cref{def:nice-ordering}, so that, since $a, b$ are consecutive
  in $\sigma$ and belong to the same $L_i(\sigma)$, any parent of $a$ must also
  be a parent of $b$.  We thus conclude that $a \rightarrow b$ cannot be
  strongly ${\mathcal{I}}$-protected in $\mathcal{E}_{\mathcal{I}}(D)$, and
  hence is not directed in it.
\end{proof}

\begin{theorem}
\label{theorem:main-theorem}
  Let $D$ be a DAG without v-structures with $n$ nodes. Then, any set $I$ of
  atomic interventions that fully orients $\skel{D}$ when the ground-truth DAG
  is $D$ must be of size at least $\ceil{\frac{n-r}{2}}$ ($\skel{D}$ is also the
  MEC of $D$, since $D$ has no v-structures).  Here $r$ is the number
  of distinct maximal cliques in $\skel{D}$.  In other words, if $I$ is a set
  of atomic interventions such that $\mathcal{E}_I(D) = D$, then
  $\abs{I} \geq \ceil{\frac{n-r}{2}}$.
\end{theorem}

\begin{proof} \cref{lem:technical-lower-bound} implies that any set $I$ of
  atomic interventions that fully orients $D$ starting with $\skel{D}$ (i.e.,
  for which $\mathcal{E}_I(D) = D$) must contain at least one node of each pair
  of consecutive nodes (in $\sigma$) of $L_i(\sigma)$, for each $i \in
  [r]$. Thus, for each $i \in [r]$, $I$ must contain at least
  $\ceil{(\abs{L_i(\sigma)} - 1)/2}$ nodes of $L_i(\sigma)$. We therefore have,
  \begin{align} \abs{I} &\geq \sum_{i = 1}^r \ceil{\frac{|L_i(\sigma)| - 1}{2}}
              \geq \ceil{\sum_{i = 1}^r \frac{|L_i(\sigma)| - 1}{2}}\nonumber\\
            &= \ceil{\frac{\sum_{i = 1}^r |L_i(\sigma)|}{2} - \frac{r}{2}}
              = \ceil{\frac{n - r}{2}}.\qedhere\nonumber
  \end{align}
\end{proof}
The following corollary for general DAGs (those that may have v-structures)
follows from the previous result about DAGs without v-structures in a manner
identical to previous work~\citep{squires2020}, using the fact that it is
necessary and sufficient to separately orient each chordal chain component of an
MEC in order to fully orient an MEC~\citep[Lemma 1]{hauser_two_2014}.  We defer
the standard proof to
Section~\ref{sec:supp:proof:main-arbitrary-DAGS}.
\begin{theorem}
  \label{theorem:main-theorem-arbitrary-DAGs}
  Let $D$ be an arbitrary DAG and let $\mathcal{E}(D)$ be the chain graph with
  chordal chain components representing the MEC of $D$. Let $CC$
  denote the set of chain components of $\mathcal{E}(D)$, and $r(S)$ the
  number of maximal cliques in the chain component $S \in CC$.  Then, any set of
  atomic interventions which fully orients $\mathcal{E}(D)$ must be of size at
  least
  \begin{displaymath}
\sum_{S \in CC}\ceil{\frac{\abs{S} - r(S)}{2}} \geq \ceil{\frac{n - r}{2}},
\end{displaymath}
  where $n$ is the number of nodes in $D$, and $r$ is the total number of maximal
  cliques in the chordal chain components of $\mathcal{E}(D)$ (including chain
  components consisting of singleton vertices).
\end{theorem}

\subsection{Tightness of Universal Lower Bound}
We now show that our universal lower bound is tight up to a factor of $2$: for
any DAG $D$, there is a set of atomic interventions of size at most twice the
lower bound that fully orients the MEC of $D$. In fact, as the proof of the
theorem below shows, when $D$ has no v-structures, this intervention set can be
taken to be the set of nodes of $D$ that are \emph{not} \sinkv{} nodes of $D$.

\begin{theorem}\label{theorem:correctness-alg-sink-nodes}
Let $D$ be a DAG without v-structures with $n$ nodes, and let $r$ be the
  number of distinct maximal cliques in $\skel{D}$.  Then, there exists a set
  $I$ of atomic interventions of size at most $n - r$ such that $I$ fully
  orients \skel{D} (i.e., $\mathcal{E}_I(D) = D$).
\end{theorem}
\begin{proof}
  Fix any topological ordering $\sigma$ of $D$. Let the maximal cliques of $D$
  be $C_1, \dots, C_r$, and let $s_i \defeq \sink[D]{C_i}$, for $i \in
  [r]$. \Cref{lemma:sink-nodes} implies that each node of
  $S = \{s_1, \dots, s_r\}$ is distinct.  We re-index these nodes according to
  the ordering $\sigma$, i.e.  $\sigma(s_i) < \sigma(s_j)$ when $i < j$.
  Consider the set $I \defeq V \setminus S$ of atomic interventions (note that
  $\abs{I} = n - r$).  We show that $\mathcal{E}_I(D) = D$.  Note that every
  edge of $D$, except those which have both end-points in $S$, has a single
  end-point in one of the interventions in $I$, and hence is directed in
  $\mathcal{E}_I(D)$ (by \cref{item:directed-by-intervention} of
  \Cref{thm:hb-i-essential}).  We show now that all edges with both end-points
  in $S$ are also oriented in $\mathcal{E}_I(D)$.

  Suppose, if possible, that there exist $s_i, s_j \in S$, with $i < j$ such
  that $s_i$ and $s_j$ are adjacent in \skel{D}, so that the edge
  $s_i \rightarrow s_j$ is present in $D$, but for which $s_i - s_j$ is not
  directed in $\mathcal{E}_I(D)$.  We derive a contradiction to this supposition.
  To start, choose $s_i, s_j$ as above with the smallest possible value of
  $i$.  In particular, this choice implies that every edge of the form
  $u \rightarrow s_i$ in $D$ is directed in $\mathcal{E}_I(D)$.

  Note that, by \cref{lemma:sink-nodes}, $C_i = \inb{s_i} \cup \pa[D]{s_i}$ and
  $C_j = \inb{s_j} \cup \pa[D]{s_j}$ are distinct maximal cliques in $\skel{D}$.
Thus,
  there must exist an $x \in C_i$ that is not a parent of $s_j$ in $D$.
  Further, since $\sigma(s_i) < \sigma(s_j)$, all vertices of $C_i$ appear
  before $s_j$ in $\sigma$.  Thus, $x \in C_i$ that is not a parent of $s_j$ in
  $D$ is also not adjacent to $s_j$ in \skel{D}.  Further, by the choice of $i$,
  the edge $x \rightarrow s_i$ is directed in $\mathcal{E}_I(D)$.  Thus, we have
  the induced subgraph $x \rightarrow s_i - s_j$ in $\mathcal{E}_I(D)$.
  However, according to \cref{item:directed-by-Meek-rule-1} of
  \cref{thm:hb-i-essential}, such a graph cannot appear as an induced subgraph
  of an $I$-essential graph $\mathcal{E}_I(D)$, and we have therefore reached
  the desired contradiction.  It follows that $\mathcal{E}_I(D)$ has no
  undirected edges, and is therefore the same as $D$.
\end{proof}

Using again the fact that it is necessary and sufficient to separately orient
each of the chordal chain components of an MEC in order to fully orient an MEC,
the following result for general DAGs follows immediately from
\cref{theorem:correctness-alg-sink-nodes}, and implies that the lower bound for
general DAGs is also tight up to a factor of $2$ (see
Section~\ref{sec:supp:general-DAGs-tight-lower-bound} for a detailed statement
and proof).
\begin{theorem}\label{thm:tight-lower-bound-general}
  Let $D$ be an arbitrary DAG on $n$ nodes and let $\mathcal{E}(D)$ and $r$ be as in the
  notation of \cref{theorem:main-theorem-arbitrary-DAGs}.  Then, there is a set
  of atomic interventions of size at most $n - r$ that fully orients
  $\mathcal{E}(D)$.
\end{theorem}

\subsection{Comparison with Known Lower Bounds}
\label{subsec:comparison-with-known-lb}
To compare our universal lower bound with the universal lower bound of
\cite{squires2020}, we start with the following combinatorial lemma, whose proof
can be found in
Section~\ref{sec:supp:proof-clique-number-bound}.
\begin{lemma} \label{lemma:clique-number-bound} Let $G$ be an undirected chordal
  graph on $n$ nodes in which the size of the largest clique is $\omega$.  Then,
  $n - |\mathcal{C}| \geq \omega - 1$, where $\mathcal{C}$ is the set of
  maximal cliques of $G$.
\end{lemma}

\Cref{lemma:clique-number-bound} implies that
$\ceil{\frac{n - |\mathcal{C}|}{2}} \geq \ceil{\frac{\omega - 1}{2}} =
\floor{\frac{\omega}{2}}$ in chordal graphs which shows that our universal lower
bound is always equal to or better than the one by \cite{squires2020}.  The
proof of \cref{lemma:clique-number-bound} makes it apparent that two bounds
are close only in very special circumstances. (\emph{Split graphs} and
\emph{k-trees} are some special families of chordal graphs for which
$\ceil{\frac{n - \abs{\mathcal{C}}}{2}} = \floor{\frac{\omega}{2}}$). We further
strengthen this intuition through theoretical analysis of special classes of
graphs and via simulations.

\paragraph{Examples where our Lower Bound is Significantly Better}
\label{subsubsec:examples}
We provide two constructions of special classes of chordal graphs in which our
universal lower bound is $\Theta(k)$ times the $\floor{\frac{\omega}{2}}$ lower
bound by \cite{squires2020} for any $k \in \mathbb{N}$. Further discussion of
such examples can be found in
Section~\ref{sec:supp:graph-examples}.

\textit{Construction 1.} First, we provide a construction by
\cite{shanmugamKDV15} for graphs that require about $k$ times more number of
interventions than their lower bound, where $k$ is size of the maximum
independent set of the graph. This construction of a chordal graph $G$ starts
with a line $L$ consisting of vertices $1, \dots, 2k$ such that each node
$1 < i < 2k$ is connected to $i-1$ and $i+1$. For each $1 \leq p \leq k$, $G$
has a clique $C_p$ of size $\omega$ which has exactly two nodes $2p-1, 2p$ from
the line $L$. Maximum clique size of $G$ is $\omega$, number of nodes,
$n = k\omega$, and number of maximal cliques, $\abs{\mathcal{C}} = 2k -
1$. Thus, for $G$, we have, $n - \abs{\mathcal{C}} = k(\omega - 2) + 1$ which
implies
$\ceil{\frac{n - \abs{\mathcal{C}}}{2}} = \Theta(k)\floor{\frac{\omega}{2}}$ for
$\omega > 2$.

\textit{Construction 2.} $G$ has $k$ cliques of size $\omega$, with every pair of cliques intersecting at a
unique node $v$. The number of nodes in $G$ is $k(\omega - 1) + 1$, maximum
clique size is $\omega$, and number of maximal cliques is $k$, thus,
$n - \abs{\mathcal{C}} = k(\omega - 2) + 1$ which implies
$\ceil{\frac{n - \abs{\mathcal{C}}}{2}} = \Theta(k)\floor{\frac{\omega}{2}}$ for
$\omega > 2$.

\section{Multi Node Interventions}
\label{sec:multi-node-interventions}

In this section, we explore the applicability of the techniques developed in the
previous section to the setting of non-atomic interventions, where each
intervention can potentially randomize more than one node.  We begin with the
lower bound in \cref{thm:lower-bound-size-k-interventions}, which is a direct
corollary of our lower bound for single-node interventions.  Then, in
\cref{thm:upper-bound-size-k-interventions}, we explore how tight this lower
bound is.  The technical core of this result is
\cref{thm:upper-bound-single-intervention}, which uses some of the ideas behind
the notion of \gcb{} orderings.
\begin{proposition}\label{thm:lower-bound-size-k-interventions}
  Let $D$ be an arbitrary DAG on $n$ nodes and let $\mathcal{E}(D)$ and $r$ be as in the
  notation of \cref{theorem:main-theorem-arbitrary-DAGs}. If $\mathcal{I}$ is a set of
  interventions of size at most $k$ that fully orients $\mathcal{E}(D)$
  (i.e., $\mathcal{E}_{\mathcal{I}}(D) = D$), then
  $\abs{\mathcal{I}} \geq \ceil{\frac{\ceil{\frac{n-r}{2}}}{k}}$.
\end{proposition}
\begin{proof}[of \cref{thm:lower-bound-size-k-interventions}]
  Let $\mathcal{I}$ be a set of interventions of size at most $k$ such that
  $\mathcal{E}_{\mathcal{I}}(D) = D$.  Consider the set $\mathcal{I}^*$ of
  single node interventions obtained from $\mathcal{I}$ by braking each
  intervention $I \in \mathcal{I}$ into its constituent vertices.  Formally,
  \begin{displaymath}
    \mathcal{I}^* \defeq \inb{\emptyset} \cup \inb{\inb{v} \st v \in I \text{ for
        some } I \in
      \mathcal{I}}.
  \end{displaymath}
  Then, (recalling that the empty intervention is not counted when reporting the
  size of the intervention set), we have
  $\abs{\mathcal{I}^*} \leq k \abs{\mathcal{I}}$, since each intervention in
  $\mathcal{I}$ is of size at most $k$.  Further, from
  \cref{cor:partition-intervention} in Section~\ref{sec:some-folkl-results}, we
  also have $\mathcal{E}_{\mathcal{I}^*}(D) = D$ (since
  $\mathcal{E}_{I}(D) = D$).

  However, from \cref{theorem:main-theorem-arbitrary-DAGs} we know that since
  $\mathcal{I}^*$ is a set of single node interventions which fully orients
  $\mathcal{E}(D)$, we must have
  $\abs{\mathcal{I}^*} \geq \ceil{\frac{n-r}{2}}$. Thus, we get that
  \[\abs{\mathcal{I}} \geq \ceil{\frac{\abs{\mathcal{I}^*}}{k}} =
    \ceil{\frac{\ceil{\frac{n-r}{2}}}{k}}.\] We conclude that any intervention
  set of interventions of size at most $k$ that fully orients $\mathcal{E}(D)$
  must have at least $\ceil{\frac{\ceil{\frac{n-r}{2}}}{k}}$ interventions.
\end{proof}

In \cref{thm:upper-bound-size-k-interventions} below we show that the above
lower bound is tight up to an additive term of roughly $R/k$, where $R$ is the
total number of maximal cliques in all ``non-singleton'' chordal chain
components of the MEC, and $k$ is the maximum size of each intervention.  The
main technical ingredient of that result is the theorem below for a single chain
component.

\begin{theorem}\label{thm:upper-bound-single-intervention}
  Let $D$ be an DAG on $n$ nodes without v-structures and let $\sigma$ be a
  topological ordering of $D$ that satisfies the clique block property P1 of
  \cref{def:nice-ordering}.  Suppose that $D$ has $r$ \sinkv{} nodes, and let
  $L_i(\sigma)$, $1 \leq i \leq r$ be as in P1 of \cref{def:nice-ordering}.
  Then, there is a set $I$ of at most
  $\sum_{i = 1}^r\ceil{\frac{\abs{L_i(\sigma)}}{2}}$ nodes such that the
  intervention set $\mathcal{I} \defeq \inb{\emptyset, I}$ fully orients $D$
  starting from its MEC (which is the same as \skel{D} as $D$ has no
  v-structures).  In other words, $\mathcal{E}_{\mathcal{I}}(D) = D$.
\end{theorem}

\begin{proof}[Proof of \cref{thm:upper-bound-single-intervention}]
  The proof is similar to that of \cref{theorem:correctness-alg-sink-nodes}.
  Let $L_i(\sigma)$, for $1 \leq i \leq r$ (where $r$ is the number of \sinkv{}
  nodes of $D$), be as in P1 of \cref{def:nice-ordering}.  For each
  $1 \leq i \leq r$, order the nodes of $L_i(\sigma)$ as $L_i(\sigma)_j$,
  $1 \leq j \leq \abs{L_i(\sigma)}$, according to $\sigma$.  For each node $v$,
  let $p(v)$ denote the parent of $v$ that has the highest index according to
  $\sigma$ (i.e., such that $\sigma(u) \leq \sigma(p(v))$ for every parent $u$ of
  $v$).  $I$ is defined using the following procedure:

  \begin{procedure}[H]
    \SetAlgoLined
    \LinesNumbered
    \KwIn{A DAG $D$ without v-structures}
    \KwOut{A set $I$ of nodes such that $\mathcal{E}_{\inb{\emptyset, I}}(D) = D$}

    $I \leftarrow \emptyset$\;
    $\sigma \leftarrow$ A topological ordering of $D$ satisfying the clique block property P1 of \cref{def:nice-ordering}\;
    \For{$j \in [r]$}{
      \uIf{$j = 1$ \textbf{\textup{or}} $p(L_j(\sigma)_1) \notin I$\label{line:j1}} {
        $I \leftarrow  I \cup \inb{L_j(\sigma)_{2\ell - 1} | 1 \leq \ell \leq \ceil{\frac{\abs{L_j(\sigma)}}{2}}}$\;
      } \uElse{ $I \leftarrow I \cup \inb{L_j(\sigma)_{2\ell} | 1 \leq \ell \leq \floor{\frac{\abs{L_j(\sigma)}}{2}}}$\;
      }
    }
    \Return $I$\;
\end{procedure}

  By construction, $\abs{I} \leq \sum_{i=1}^r\ceil{\abs{L_i(\sigma)}/2}$.  We
  will now complete the proof by showing that
  $\mathcal{E}_{\mathcal{I}}(D) = D$, where
  $\mathcal{I} \defeq \inb{\emptyset, I}$.  To do this, we will show that any
  edge $a \rightarrow b$ in $D$ is directed in $\mathcal{E}_{\mathcal{I}}(D)$.

  Consider first the case when $a$ and $b$ are in the same $L_i(\sigma)$.  If
  they are also consecutive in $\sigma$, then, by construction,
  $\abs{I \cap \inb{a, b}} = 1$, so that the edge $a \rightarrow b$ would be
  directed in $\mathcal{E}_{\mathcal{I}}(D)$ (by
  \cref{item:directed-by-intervention} of \cref{thm:hb-i-essential}).  Now,
  suppose that $a, b$ are in the same $L_i(\sigma)$, but are not consecutive in
  $\sigma$.  Let $v_1, v_2, \dots v_\ell$ be the vertices (all of which must be
  in $L_i(\sigma)$) between $a$ and $b$ in $\sigma$, ordered according to
  $\sigma$.  By the previous argument, $a \rightarrow v_1$,
  $v_i \rightarrow v_{i + 1}$ (for $1 \leq i \leq \ell - 1$) and
  $v_{\ell} \rightarrow b$ are all directed edges in
  $\mathcal{E}_{\mathcal{I}}(D)$ (here we are also using the fact that
  $L_i(\sigma)$ is a clique in \skel{D}).  But then, if $a \rightarrow b$ were
  not directed in $\mathcal{E}_{\mathcal{I}}(D)$, the undirected edge
  $a \undir b$ along with the above directed path from $a$ to $b$ would form a
  directed cycle in $\mathcal{E}_{\mathcal{I}}(D)$, contradicting
  \cref{item:chain-chordal} of \cref{thm:hb-i-essential} (which says that
  $\mathcal{E}_{\mathcal{I}}(D)$ must be a chain graph).  Thus, we conclude that
  if $a$ and $b$ are in the same $L_i(\sigma)$, then $a \rightarrow b$ is
  directed in $\mathcal{E}_{\mathcal{I}}(D)$.

  Now, suppose, if possible, that there is some edge $a \rightarrow b$ in $D$
  which is not directed in $\mathcal{E}_{\mathcal{I}}(D)$.  Among all such
  edges, choose the one for which the pair $(a, b)$ is lexicographically
  smallest according to $\sigma$.  By the previous paragraph, $a$ and $b$ cannot
  belong to the same $L_i(\sigma)$, so that there must exist $j > i$ such that
  $a \in L_i(\sigma)$ and $b \in L_j(\sigma)$.  We consider the possibilities
  $b = L_j(\sigma)_1$ and $b \neq L_j(\sigma)_1$ separately.

  Suppose first, if possible, that $b \neq L_j(\sigma)_1$.  Then, there must
  exist a vertex $c \in L_j(\sigma)$, such that $c \rightarrow b$ is directed in
  $\mathcal{E}_{\mathcal{I}}(D)$ (here again we use the fact that $L_i(\sigma)$
  is a clique).  Now, $a$ must be adjacent to $c$ in \skel{D}, for otherwise,
  $\mathcal{E}_{\mathcal{I}}(D)$ would contain the induced graph
  $c \rightarrow b \undir a$, contradicting \cref{item:directed-by-Meek-rule-1}
  of \cref{thm:hb-i-essential}.  Further, the edge $a \rightarrow c$ must be
  directed in $\mathcal{E}_{\mathcal{I}}(D)$, by the choice of $(a, b)$ as the
  lexicographically smallest undirected edge in $\mathcal{E}_I(D)$ (for
  otherwise, $(a, c)$ would be a lexicographically smaller undirected edge
  compared to $(a, b)$, since $c$ comes before $b$ in $\sigma$).  But we then
  have the directed cycle $a \dir c \dir b \undir a$ in
  $\mathcal{E}_{\mathcal{I}}(D)$, thereby contradicting
  \cref{item:chain-chordal} of \cref{thm:hb-i-essential} (which says that
  $\mathcal{E}_{\mathcal{I}}(D)$ must be a chain graph).  Thus, we conclude that
  $b \neq L_j(\sigma)_1$ is not possible.

  This leaves only the possibly that $b = L_j(\sigma)_1$.  We now show that this
  also leads to contradiction.  Consider first the case that $b \in I$.  Then,
  since the edge $a \rightarrow b$ is undirected in
  $\mathcal{E}_{\mathcal{I}}(D)$, we must have $a \in I$ also (by
  \cref{item:directed-by-intervention} of \cref{thm:hb-i-essential}).  From the
  construction of $I$, it then follows that $p(b) \neq a$ and also that
  $p(b) \not\in I$ (to see this consider line~\ref{line:j1} of the process
  above, and note that by our choice of $a$ and $b$, $b = L_j(\sigma)_1$ and
  $a \in L_i(\sigma)$ with $j > i \geq 1$).  Thus,
  $\abs{I \cap \inb{a, p(b)}} = 1$ and $\abs{I \cap \inb{b, p(b)}} = 1$.
  Similarly, if $b \not\in I$, then, since the edge $a \rightarrow b$ is
  undirected in $\mathcal{E}_{\mathcal{I}}(D)$, we must have $a \not\in I$ also
  (by \cref{item:directed-by-intervention} of \cref{thm:hb-i-essential}).  By
  the condition on line~\ref{line:j1} of the construction of $I$, we must then
  have $p(b) \in I$ (since $b \not \in I$), and hence again that $a \neq p(b)$
  and $\abs{I \cap \inb{a, p(b)}} = 1$, $\abs{I \cap \inb{b, p(b)}} = 1$.

  Thus, we obtain that if $b = L_j(\sigma)_1$, then $a \neq p(b)$, and further
  that $\abs{I \cap \inb{a, p(b)}} = 1$ and $\abs{I \cap \inb{b, p(b)}} = 1$.
  Since $D$ has no v-structures, we must have the edge $a \rightarrow p(b)$
  in $D$ (by the definition of $p(b)$), and also in $\mathcal{E}_{\mathcal{I}}(D)$
  (by \cref{item:directed-by-intervention} of \cref{thm:hb-i-essential}, since
  $\abs{I \cap \inb{a, p(b)}} = 1$). Also $p(b) \rightarrow b$ in
  $\mathcal{E}_{\mathcal{I}}(D)$ (by \cref{item:directed-by-intervention} of
  \cref{thm:hb-i-essential}, since $\abs{I \cap \inb{b, p(b)}} = 1$). But then
  $\mathcal{E}_{\mathcal{I}}(D)$ contains the directed cycle
  $a \rightarrow p(b) \rightarrow b \undir a$, which contradicts
  \cref{item:chain-chordal} of \cref{thm:hb-i-essential} (which says that
  $\mathcal{E}_{\mathcal{I}}(D)$ must be a chain graph).  We conclude that
  $b = L_j(\sigma)_1$ is also not possible.  This concludes the proof.
\end{proof}

\begin{corollary}\label{cor:upper-bound-single-intervention-general}
  Let $D$ be an arbitrary DAG and let $\mathcal{E}(D)$ be the chain graph with
  chordal chain components representing the MEC of $D$. Let $CC$ denote the set
  of chain components of $\mathcal{E}(D)$, and $r(S)$ the number of maximal
  cliques in the chain component $S \in CC$. Let $\sigma(S)$ be a topological
  ordering of $D[S]$ that satisfies the clique block property P1 of
  \cref{def:nice-ordering} (note that $D[S]$ is a DAG without v-structures since
  $S$ is a chain component of $\mathcal{E}(D)$). Let $L_i(\sigma(S))$,
  $1 \leq i \leq r(S)$ be as in P1 of \cref{def:nice-ordering}. Then, there is a
  set $I$ of at most
  $\sum_{S \in CC} \mathbb{I}(\abs{S} > 1) \sum_{i = 1}^{r(S)}
  \ceil{\frac{\abs{L_i(\sigma(S))}}{2}}$ nodes such that the intervention set
  $\mathcal{I} \defeq \inb{\emptyset, I}$ fully orients $D$ starting from its
  MEC. In other words, $\mathcal{E}_{\mathcal{I}}(D) = D$.
\end{corollary}
\begin{proof}[Proof of \cref{cor:upper-bound-single-intervention-general}]
  We use \cref{corollary:supp:HB14Lemma1}, which says that it is sufficient to
  orient all chain components of $\mathcal{E}(D)$ in order to fully orient
  $\mathcal{E}(D)$.  Formally, for each $S \in CC$, $D[S]$ is a DAG without
  v-structures.  Thus, for each $S \in CC$ with $\abs{S} > 1$, we see from
  \cref{thm:upper-bound-single-intervention} that there is a set
  $I_S \subseteq S$ of size at most
  $ \sum_{i = 1}^{r(S)} \ceil{\frac{\abs{L_i(\sigma(S))}}{2}}$ such
  $\mathcal{E}_{\inb{\emptyset, I_S}}(D[S]) = D[S]$.  We now define
  \begin{equation}
    I \defeq \bigcup_{\substack{S \in CC\\\abs{S} > 1}} I_S \qquad \text{ and }
    \qquad \mathcal{I} \defeq \inb{\emptyset, I}.\label{eq:3}
  \end{equation}
  In the notation of \cref{corollary:supp:HB14Lemma1}, we thus have
  $\mathcal{I}_S = \inb{\emptyset, S}$ for every $S \in CC$ with $\abs{S} > 1$.
  Thus from the above discussion, we have
  $\mathcal{E}_{\mathcal{I}_S}(D[S]) = D[S]$ for every $S \in CC$ (this equation
  is trivially true for those $S \in CC$ for which $\abs{S} = 1$, since in that
  case $D[S]$ does not have any edges to
  direct). \cref{corollary:supp:HB14Lemma1} thus implies that
  $\mathcal{E}_{\mathcal{I}}(D) = D$.  The claim now follows from the
  definitions of $I$ and $I_S$ above.
\end{proof}

\begin{theorem}\label{thm:upper-bound-size-k-interventions}
  Let $D$ be an arbitrary DAG on $n$ nodes and let $\mathcal{E}(D)$ and $r$ be as in the
  notation of \cref{theorem:main-theorem-arbitrary-DAGs}. Then, there exists a set
  $\mathcal{I}$ of interventions of size at most $k$ such that $\mathcal{I}$ fully orients
  $\mathcal{E}(D)$ (i.e., $\mathcal{E}_{\mathcal{I}}(D) = D$) and
  $\abs{\mathcal{I}} \leq \ceil{\frac{\floor{\frac{n-r}{2}}}{k}} + \ceil{\frac{r^*}{k}}$. Here, $r^*$
  is the total number of maximal cliques in the chordal chain components of $\mathcal{E}(D)$,
  excluding chain components of singleton vertices.
\end{theorem}
\begin{proof}[Proof of \cref{thm:upper-bound-size-k-interventions}]
  From \cref{cor:upper-bound-single-intervention-general}, we get that there
  exists a set $I$ with
  $\abs{I} \leq \sum_{S \in CC} \mathbb{I}(\abs{S} > 1) \sum_{i = 1}^{r(S)}
  \ceil{\frac{\abs{L_i(\sigma(S))}}{2}}$ such that
  $\mathcal{I}^* \defeq \inb{\emptyset, I}$ fully orients $D$ starting from
  $\mathcal{E}(D)$.  Now, write the elements of $I$ as
  $I_1, I_2, \dots, I_{\abs{I}}$, and define the partitions
  \[I^i \defeq \inb{I_{(i-1)k+1}, \dots, I_{\max(ik, \abs{I})}},\] for
  $i \in \inb{1, 2, \dots, \ceil{\frac{\abs{I}}{k}}}$.  Then,
  \cref{cor:partition-intervention} implies that if
  $\mathcal{I}^* = \inb{\emptyset, I}$ fully orients $\mathcal{E}(D)$, then
  $\mathcal{I} \defeq \inb{\emptyset, I^1, I^2, \dots,
    I^{\ceil{\frac{\abs{I}}{k}}}}$ also fully orients $\mathcal{E}(D)$.  Thus,
  we conclude that $\mathcal{I}$ indeed fully orients $\mathcal{E}(D)$ (i.e.,
  $\mathcal{E}_{\mathcal{I}}(D) = D$). Now, we have,
  \begin{align} \abs{I} &\leq \sum_{S \in CC} \mathbb{I}(\abs{S} > 1) \sum_{i = 1}^{r(S)} \ceil{\frac{\abs{L_i(\sigma(S))}}{2}}\nonumber\\
            &= \sum_{S \in CC} \mathbb{I}(\abs{S} > 1) \sum_{i = 1}^{r(S)} \left(\floor{\frac{\abs{L_i(\sigma(S))} - 1}{2}} + 1\right)\nonumber\\
            &\leq \sum_{S \in CC} \mathbb{I}(\abs{S} > 1) \inp{\floor{\sum_{i = 1}^{r(S)} \frac{\abs{L_i(\sigma(S))} - 1}{2}} + r(S)}\nonumber\\
            &= \sum_{S \in CC} \mathbb{I}(\abs{S} > 1) \inp{\floor{\frac{\abs{S} - r(S)}{2}} + r(S)}\nonumber\\
            &\leq \floor{\sum_{S \in CC} \frac{\abs{S} - r(S)}{2}} + r^*
            = \floor{\frac{n- r}{2}} + r^*.\nonumber
\end{align}
  Therefore, recalling that the empty intervention is not counted when
  reporting the size of an intervention set, we see that the size of the
  intervention set $\mathcal{I}$ satisfies
  \begin{displaymath} \abs{\mathcal{I}} = \ceil{\frac{\abs{I}}{k}}
            \leq \ceil{\frac{\floor{\frac{n- r}{2}}}{k} + \frac{r^*}{k}}
    \leq \ceil{\frac{\floor{\frac{n- r}{2}}}{k}} + \ceil{\frac{r^*}{k}}.\nonumber
  \end{displaymath}
  Thus, we get that $\abs{\mathcal{I}}$ is at most $\ceil{\frac{r^*}{k}}$ more
  than the lower bound of $\ceil{\frac{\ceil{\frac{n- r}{2}}}{k}}$ imposed by
  \cref{thm:lower-bound-size-k-interventions} on the smallest set of
  interventions of size at most $k$ that fully orients $D$ starting from
  $\mathcal{E}(D)$. This completes the proof.
\end{proof}

\section{Empirical Explorations}
\label{sec:empir-expl}
In this section, we report the results of two experiments on synthetic data. In
\emph{Experiment $1$},
we compare our lower bound with the \emph{optimal intervention size} for a large
number of randomly generated DAGs. Optimal intervention size for a DAG $D$ is
defined as the size of the smallest set of atomic interventions $I$ such that
$\mathcal{E}_I(D) = D$. Next, in \emph{Experiment $2$}, we compare our universal
lower bound with the one in the work of \cite{squires2020} for randomly generated DAGs with
small cliques. These experiments provide empirical evidence that strengthens our
result about the tightness of our universal lower bound
(\Cref{thm:tight-lower-bound-general}) and the constructions presented in
\Cref{subsubsec:examples}. The experiments use the open source
\texttt{causaldag}~\citep{squires18:_python} and
\texttt{networkx}~\citep{hagberg08:_explor_networ} packages.  Further details
about the experimental setup for both experiments are given in
Section~\ref{sec:supp:experiments}.

\paragraph{Experiment 1}
\begin{figure}[ht]
    \centering
    \includegraphics[scale=0.7]{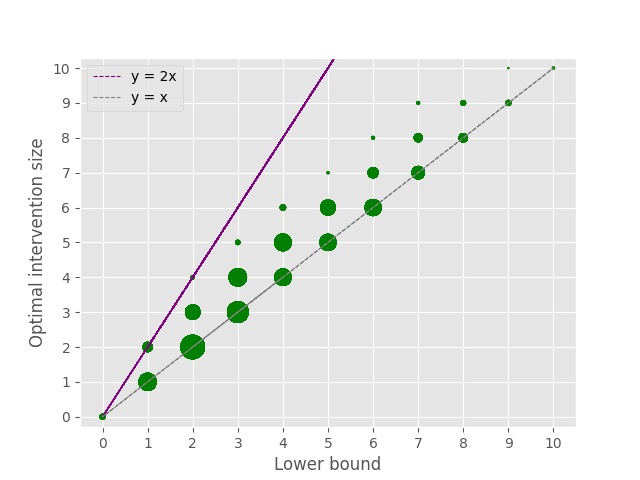}
    \vspace{\baselineskip}
    \caption{Comparison of the Optimal Intervention Set Size with our Universal
      Lower Bound}
    \label{fig:opt-vs-lower-bound}
\end{figure}
For this experiment, we generate $1000$ graphs from Erd\H{o}s-R\'enyi graph model $G(n,p)$:
for each of these graphs, the number of nodes $n$ is a random integer in $[5, 25]$ and
the connection probability $p$ is a random value in $[0.1, 0.3)$. These graphs
are then converted to DAGs without v-structures by imposing a random topological ordering
and adding extra edges if needed. To compute the optimal intervention size, we check
if a subset of nodes, $I$ of a DAG $D$ is such that $\mathcal{E}_{\mathcal{I}}(D) = D$, in
increasing order of the size of such subsets. Next, we compute the universal lower bound value
for each of these DAGs as given in \Cref{theorem:main-theorem-arbitrary-DAGs}.
In \Cref{fig:opt-vs-lower-bound}, we plot the optimal intervention size and our lower bound
for each of the generated DAGs. Thickness of the points is proportional to the number of
points landing at a coordinate. Notice that, all points lie between lines $y = x$ and $y = 2x$,
as implied by our theoretical results. Further, we can see that, a large fraction of points
are closer to the line $y = x$ compared to the line $y = 2x$, suggesting that our lower bound
is even tighter for many graphs.

\paragraph{Experiment 2}
\begin{figure}[ht]
    \centering
    \includegraphics[scale=0.7]{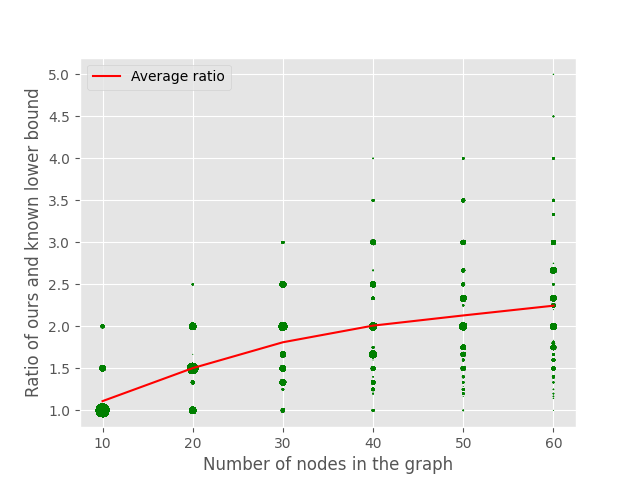}
    \vspace{\baselineskip}
    \caption{Comparison of our Universal Lower Bound with that of \cite{squires2020}}
    \label{fig:comparison-lower-bounds}
\end{figure}
For this experiment, we generate $1000$ random DAGs without v-structures for each size in
$\inb{10, 20, 30, 40, 50, 60}$ by fixing a perfect-elimination ordering of the
nodes and then adding edges (which are oriented according to the
perfect-elimination ordering) to the DAG making sure that there are no
v-structures, while trying to keep the size of each clique below $5$.
For each DAG, we compute the ratio of the two lower bounds. In
\Cref{fig:comparison-lower-bounds}, we plot each of these ratios in a scatter
plot with the $x$-axis representing the number of nodes of the DAG. Thickness of
the points is proportional to the number of DAGs having a particular value of
the ratio described above. We also plot the average of the ratios for each
different value of the number of nodes. We see that our lower bound can
sometimes be $\sim 5$ times of the lower bound of \cite{squires2020}.
Moreover, the average ratio has an increasing trend suggesting that our lower
bound is much better for this class of randomly generated DAGs.

\section{Conclusion}
We prove a strong universal lower bound on the minimum number of atomic
interventions required to fully learn the orientation of a DAG starting from its
MEC. For any DAG $D$, by constructing an explicit set of atomic interventions
that learns $D$ completely (starting with the MEC of $D$) and has size at most
twice of our lower bound for the MEC of $D$, we show that our universal lower
bound is tight up to a factor of two. We prove that our lower bound is better
than the best previously known universal lower bound \citep{squires2020} and
also construct explicit graph families where it is significantly better.  We
then provide empirical evidence that our lower bound may be stronger than what
we are able to prove about it: by conducting experiments on randomly generated
graphs, we demonstrate that our lower bound is often tighter (than what we have
proved), and also that it is often significantly better than the previous
universal lower bound \citep{squires2020}. We then illustrate that the notion of
\gcb{} orderings is also useful in handling the case of multi-node
interventions. An interesting direction for future work is to design
intervention sets of sizes close to our universal lower bound.  We note that in
contrast to the earlier work of \cite{squires2020}, whose lower bound proofs
were based on new sophisticated constructions, our proof is based on the simpler
notion of a \gcb{} ordering, which in turn is inspired from elementary ideas in
the theory of chordal graphs.  We expect that the notion of \gcb{} orderings may
also play an important role in future work on designing optimal intervention
sets.

\acks{We thank AISTATS reviewers for their comments and suggestions. PS
  acknowledges support from the Department of Atomic Energy, Government of
  India, under project no. RTI4001, from the Ramanujan Fellowship of SERB, from
  the Infosys foundation, through its support for the Infosys-Chandrasekharan
  virtual center for Random Geometry, and from Adobe Systems Incorporated via a
  gift to TIFR. The contents of this paper do not necessarily reflect the views
  of the funding agencies listed above.}

\newpage

\appendix
\section{Properties of \GCB{} Orderings: Omitted Proofs}
In this section, we provide the omitted proofs of the various properties of
\gcb{} orderings.

\subsection{Proof of \cref{lemma:sink-nodes}}
\label{sec:supp:proof-sink-nodes}
\begin{proof}[Proof of \cref{lemma:sink-nodes}]
  Let $C$ be a maximal clique of $\skel{D}$.  Since the induced subgraph $D[C]$
  is a DAG, there is at least one node $s$ in $D[C]$ with out-degree $0$.  Thus,
  for all $v \in C, v \neq s$ we have, $v \rightarrow s$, which implies that
  $C \setminus \inb{s} \subseteq \pa[D]{s}$. Now, $\pa[D]{s} \cup \inb{s}$ must
  be a clique as $D$ does not contain v-structures.  Thus, we must indeed have
  $\pa[D]{s} = C \setminus \inb{s}$ since $C$ is maximal.  We thus see that
  there is a unique $s \in C$ such that $C = \pa[D]{s} \cup \inb{s}$.

  Now, suppose, if possible that there exist distinct maximal cliques $C_1$ and
  $C_2$ of $\skel{D}$ such that $\sink[D]{C_1} = \sink[D]{C_2} = s$.  Since
  $C_1$ and $C_2$ are distinct maximal cliques, there must exist
  $a \in C_1 \setminus C_2, b \in C_2 \setminus C_1$ such that $a$ is not
  adjacent to $b$.  But then, since we have $a \in \pa[D]{s}$ and
  $b \in \pa[D]{s}$, we would have a v-structure $a \rightarrow s \leftarrow b$,
  which is a contradiction to the hypothesis that $D$ has no v-structures.
\end{proof}

\subsection{Proof of Lemma~\ref{lem:P1-ordering}}
\label{sec:supp:proof-p1-ordering}
We now provide the proof of \cref{lem:P1-ordering}. As stated before the
statement of the lemma, its proof follows from well-known ideas in the theory of
chordal graphs.  The following generalization of the definition of the clique
block property P1 of \cref{def:nice-ordering} will be useful in the proof.

\begin{definition}[\textbf{$A$-clique block
    ordering}]\label{supp:def:set-clique-block}
  Let $D$ be a DAG and $A$ a subset of vertices of $D$. Let $\sigma$ be a
  topological ordering of $D$.  Let the elements of $A$ be
  $a_1, a_2, \dots, a_k$, arranged so that $\sigma(a_i) < \sigma(a_j)$ whenever
  $i < j$.  Define $L_1^A(\sigma)$ to be the set of nodes $u$ which occur before
  or at the same position as $a_1$ in $\sigma$ i.e.,
  $\sigma(u) \leq \sigma(a_1)$. Similarly, for $2 \leq i \leq k$, define
  $L_i^A(\sigma)$ to be the set of nodes which occur in $\sigma$ before or at the
  same position as $a_i$, but strictly after $a_{i-1}$ (i.e.,
  $\sigma(a_{i-1}) < \sigma(u) \leq \sigma(a_i)$).
Then, $\sigma$ is said to be an \emph{$A$-clique block ordering} of $D$ if
  $\cup_{i=1}^kL_i^A(\sigma)$ is the set of all vertices of $D$, and for each
  $1\leq i \leq k$, $L_i^A(\sigma)$ is a (not necessarily maximal) clique in
  \skel{D}.
\end{definition}
The following observation is immediate with this definition.
\begin{observation}\label{supp:obv:equivalent-cb-property}
  Let $D$ be a DAG without v-structures, and let $A$ be the set of \sinkv{}
  vertices of $D$.  Then, a topological ordering $\sigma$ of $D$ satisfies
  property P1 of \cref{def:nice-ordering} if and only if $\sigma$ is an
  $A$-clique block ordering of $D$.
\end{observation}
\begin{proof}
  The ``if'' direction follows from the definition.  For the ``only if''
  direction, we note that since every vertex of $D$ must be contained in some
  maximal clique $C$ of \skel{D}, and since $C = \inb{s} \cup \pa[D]{s}$ for some
  \sinkv{} vertex $s$, it follows that every vertex of $D$ must lie in some
  $L_i(\sigma)$ if $\sigma$ satisfies the clique block property P1.
\end{proof}

We also note the following simple property of $A$-clique block orderings.
\begin{observation}\label{supp:obv:subset-property}
  Let $D$ be a DAG without v-structures, and let $\sigma$ be a topological
  ordering of $D$.  Let $A$ and $B$ be subsets of vertices of $D$ such that
  $A \subseteq B$.  If $\sigma$ is an $A$-clique block ordering of $D$, then it
  is also a $B$-clique block ordering of $D$.
\end{observation}
\begin{proof}
  When $A = B$, there is nothing to prove.  Thus, we can assume that there must
  exist a $b \in B \setminus A$.  We consider the case when $B = A \cup
  \inb{b}$. The general case then follows by straightforward induction on the size
  of $\abs{B \setminus A}$.

  For $1 \leq i \leq \abs{A}$, let $L_i^{A}(\sigma)$ be as in the definition of
  the $A$-clique block orderings.  Let $i$ be the unique index such that
  $b \in L_i^{A}(\sigma)$.  Now, for $j < i$, define
  $L_j^B(\sigma) \defeq L_j^A(\sigma)$, and for $j \geq i + 1$, define
  $L_{j+1}^B(\sigma) \defeq L_j^A(\sigma)$.  By construction, for
  $1 \leq j \leq i - 1$ and $i + 2 \leq j \leq \abs{A} + 1$, the $L_j^B(\sigma)$ are
  cliques in \skel{D}. Further define
  \begin{align*}
    L_i^B(\sigma) &\defeq \inb{u \in L_i^A(\sigma) \st \sigma(u) \leq \sigma(b)}\text{, and}\\
    L_{i+1}^B(\sigma) &\defeq L_i^A(\sigma) \setminus L_i^B(\sigma).
  \end{align*}
  Again, $L_i^B(\sigma)$ and $L_{i+1}^B(\sigma)$ are also cliques in $\skel{D}$
  since they are subsets of the clique $L_i(\sigma)$.  Further, by construction,
  $\cup_{j=1}^{|B|}L_j^B(\sigma) = \cup_{j=1}^{|A|}L_j^A(\sigma)$.  This shows
  that $\sigma$ is also a $B$-clique block ordering.
\end{proof}

We can now state the main technical lemma required for the proof of
\cref{lem:P1-ordering}.
\begin{lemma}\label{supp:lemma:ordering-technical}
  Let $D$ be a DAG without $v$-structures. Let $S$ be the set of \sinkv{}
  vertices of $D$.  Then, there exists a maximal clique $C$ of \skel{D} with the
  following two properties:
  \begin{enumerate}
  \item\label{supp:item:1} If $u \in C$ then $\pa[D]{u} \subseteq C$.  That is,
    if $u \in C$ and $v \not\in C$, then the edge $v \rightarrow u$ is not
    present in $D$.
  \item\label{supp:item:2} Let $S'$ be the set of \sinkv{} nodes of the induced DAG
    $D[V \setminus C]$, where $V$ is the set of nodes of $D$.  Then $S'$ is a
    subset of $S \setminus \inb{\sink[D]{C}}$.
  \end{enumerate}
\end{lemma}
\begin{proof}
  As already alluded to before the statement of \cref{lem:P1-ordering}, the
  proof of \cref{supp:item:1} uses ideas that are very similar to the ``maximal
  cardinality search'' algorithm for chordal graphs (\citet{tarjan_simple_1984},
  see also Corollary 2 of \cite{wienobst_polynomial-time_2020}).  Fix an
  arbitrary topological ordering $\tau$ of $D$, and let $v_1 = \tau(1)$ be the
  top vertex in $\tau$.  Note that $v_1$ has no parents in $D$, so any vertices
  adjacent to $v_1$ in $D$ are children of $v_1$.  Let $C'$ be the set of these
  children of $v_1$ in $D$.  If $C'$ is empty, then $v_1$ is isolated in $D$ and
  we are done with the proof of \cref{supp:item:1} after taking $C = {v_1}$.
  So, assume that $C'$ is not empty, and let its elements be
  $c_1, c_2, \dots c_k$, arranged so that $\tau(c_i) < \tau(c_j)$ whenever
  $i < j$.  Now, define the sets $C'_i$, where $1 \leq i \leq k$ as follows.
  First, $C_1' \defeq \inb{v_1, c_1}$.  For $2 \leq i \leq k$,
  \begin{align}
    \label{eq:1}
    C_i' \defeq
    \begin{cases}
      C_{i-1}' \cup \inb{c_i} & \text{if $ C_{i-1}' \subseteq \pa[D]{c_i}$,}\\
      C_{i-1}' & \text{otherwise}.
    \end{cases}
  \end{align}
  Define $C \defeq C_{k}'$.  Note that by construction, $C$ is a clique. Note
  also the following property of this construction: for any $i$, $c_i \not\in C$
  if and only if there exists $1 \leq j < i$ such that $c_j \in C$ and $c_j$ is
  not adjacent to $c_i$ in $D$.

  We now claim that $C$ is also a maximal clique. For, if not, let $u \not\in C$
  be such that $u$ is adjacent to every vertex in $C$.  Then, we must have
  $u = c_i$ for some $i$ (since $v_1 \in C$, and only the children of $v_1$ are
  adjacent to $v_1$ in $D$).  But then, since $u = c_i \not \in C$, there must
  exist some $c_j$, $j < i$, such that $c_j \in C$ is not adjacent to $u = c_i$,
  which is a contradiction to the assumption of $u$ being adjacent to every vertex
  of $C$.

  We now claim that if $v \not \in C$, then for all $u \in C$, the edge
  $v \dir u$ is not present in $D$.  Suppose, if possible, that there exist
  $v \not \in C$ and $u \in C$ such that $v \dir u$ is present in $D$.  By the
  choice of $v_1$ as a top vertex in a topological order, we must have
  $u \neq v_1$.  Thus, $u$ must be a child of $v_1$ in $D$.  Suppose $u = c_i$,
  for some $i \in [k]$. Then, $v$ must also be a child of $v_1$, for otherwise
  $v \dir c_i \rdir v_1$ would be a v-structure in $D$.  Thus, $v = c_j$ for
  some $j < i$.  Since $c_j = v \not\in C$, there exists some $\ell < j$ such
  that $c_\ell \in C$ and $c_\ell$ and $v = c_j$ are not adjacent.  But then
  $c_\ell \dir c_i \rdir c_j$ is a v-structure in $D$, so we again get a
  contradiction.  This proves \cref{supp:item:1} of the lemma for the clique
  $C$.

  \Cref{supp:item:2} of the lemma trivially follows if $V \setminus C$ is empty,
  therefore, we are interested in the case when $V \setminus C$ is non-empty.
  Now consider the induced DAG $H \defeq D[V \setminus C]$. Since $D$ has no
  v-structures, neither does $H$.  Thus, by \cref{lemma:sink-nodes}, the
  \sinkv{} nodes of $H$ and the maximal cliques of \skel{H} are in one-to-one
  correspondence: for each maximal clique $C'$ of \skel{H}, there is a unique
  vertex $\sink[H]{C'}$ of $H$ such that
  $C' = \inb{\sink[H]{C'}} \cup \pa[H]{\sink[H]{C'}}$.

  Consider now a \sinkv{} vertex $s'$ of $H$.  There exists then a maximal
  clique $C'$ of \skel{H} such that $s' = \sink[H]{C'}$.  Also, since $H$ is an
  induced subgraph of $D$, there must exist a maximal clique $C''$ of \skel{D}
  such that $C'' \supseteq C'$.  In fact, we must further have
  $C'' \cap (V \setminus C) = C'$, for otherwise $C'$ would not be a maximal
  clique of $H$.  Let $t \defeq \sink[D]{C''}$.  We will show that $t = s'$.
  Note first that we cannot have $t \in C$, for then, by \cref{supp:item:1},
  $C'' = \inb{t} \cup \pa[D]{t}$ would be contained in $C$, and would not
  therefore contain $C' \subseteq V \setminus C$.  Thus, $t$ must be a node in
  $V \setminus C$.  But then $C'' \cap (V \setminus C) = C'$ implies that $t$
  must in fact be in $C'$, and must therefore be equal to $s' = \sink[H]{C'}$.
  We thus see that any \sinkv{} vertex $s'$ of $H$ is also a \sinkv{} vertex of
  $D$.  \Cref{supp:item:2} of the lemma then follows by noting that
  $\sink[D]{C}$ is the only \sinkv{} vertex of $D$ not contained in
  $V \setminus C$.
\end{proof}

We are now ready to prove \cref{lem:P1-ordering}.
\begin{proof}[Proof of \cref{lem:P1-ordering}]
  We prove this claim by induction on the number of nodes in $D$.  The claim of
  the lemma is trivially true when $D$ has only one node.  Now, fix $n > 1$, and
  assume the induction hypothesis that every DAG without v-structures which has
  at most $n - 1$ nodes admits a topological ordering that satisfies the clique
  block property P1 of \cref{def:nice-ordering}.  We will complete the induction
  by showing that if $D$ is a DAG without v-structures which has $n$ nodes,
  then $D$ also admits a topological ordering that satisfies the clique block
  property P1 of \cref{def:nice-ordering}.

  Let the maximal clique $C$ of $D$ be as guaranteed by
  \cref{supp:lemma:ordering-technical} above.  If all the nodes of $D$ are
  contained in $C$, then the total ordering on the vertices of the clique $C$ in
  $D$ trivially satisfies the clique block property.  Therefore, we assume
  henceforth that $V \setminus C$ is non-empty.  Thus, the induced DAG
  $H \defeq D[V \setminus C]$ is a DAG on at most $n - 1$ nodes. Let $S'$ be the
  set of \sinkv{} nodes of $H$, and let $S$ be the set of \sinkv{} nodes of
  $D$. By the induction hypothesis, $H$ has a topological ordering $\tau$ which
  satisfies the clique block property.  Equivalently, by
  \cref{supp:obv:equivalent-cb-property}, $\tau$ is an $S'$-clique block
  ordering of $H$.

  Consider now the ordering $\sigma$ of $D$ obtained by listing first the
  vertices of the clique $C$ in the total order imposed on them by the DAG $D$,
  followed by the vertices of $V \setminus C$ in the order specified by
  $\tau$.  By \cref{supp:item:1} of
  \cref{supp:lemma:ordering-technical}, there is no directed edge in $D$ from a
  vertex in $V \setminus C$ to a vertex in $C$, so we get that $\sigma$ is in
  fact a topological ordering of $D$.

  Define $T = S' \cup {\sink[D]{C}}$.  We now observe that $\sigma$ is a
  $T$-clique block ordering of $D$, with $L_1^T(\sigma) = C$ and
  $L_{i + 1}^T(\sigma) = L_i^{S'}(\tau)$, for $1 \leq i \leq \abs{S'}$.  By
  \cref{supp:item:2} of \cref{supp:lemma:ordering-technical}, we have
  $T \subseteq S$.  Thus, by \cref{supp:obv:subset-property}, $\sigma$ is also
  an $S$-clique block ordering of $D$, and therefore (by
  \cref{supp:obv:equivalent-cb-property}) satisfies the clique block property P1
  of \cref{def:nice-ordering}.
\end{proof}

\subsection{Proof of \cref{prop:intermediate}}
\label{sec:supp:proof-prop-intermediate}
\begin{proof}[Proof of \cref{prop:intermediate}]  We use the
  same notation as in the proof of \cref{lemma:good-topological-ordering}.

  \begin{enumerate}
  \item Since $y \in Y_a$, there exists $u \in C^a$ such that $u$ is not
    adjacent to $y$ in $D$.  But then, if $z \in S^a$, we get the v-structure
    $u \rightarrow z \rdir y$, which is a contradiction to $D$ not having any
    v-structures.  This proves \cref{item:1} of
    the proposition.
  \item Consider the clique $C \defeq \inb{x} \cup \pa[D]{x}$.  There exists a
    maximal clique $C'$ in \skel{D} such that $C' \supsetneq C$, since $x$ is
    not a \sinkv{} node. Set $y \defeq \sink[D]{C'}$.  Since $x$ is not a
    \sinkv{} node in $D$, we thus have $x \in C \subseteq \pa[D]{y}$, so that
    $y$ is a child of $x$.  We also have $y \in S^a$ since
    $C^a \subseteq \inb{x} \cup \pa[D]{x} = C \subseteq \pa[D]{y}$, where the
    first inclusion comes from the assumption that $x \in \inb{a} \cup S^a$.
    The fact that $S^a$ is non-empty follows by applying the item with $x = a$,
    and noticing that, by construction, $a$ is not a \sinkv{} vertex in $D$.
    This follows since $a \in L_i(\sigma)$ has a child (namely, $b$) in
    $L_i(\sigma)$, while by the definition of the $L_i$, only the last vertex in
    $L_i(\sigma)$ is a \sinkv{} node of $D$. This proves \cref{item:2} of the
    proposition.
  \item Since $x$ is a \sinkv{} node in $H = D[S^a]$,
    $C \defeq \inb{x} \cup \pa[H]{x}$ is a maximal clique in \skel{H}, and
    thereby a clique in \skel{D}.  However, note that $C' \defeq C \cup C^a$ is
    also a clique in \skel{D}, since both $C$ and $C^a$ are cliques in \skel{D},
    and as $C \subseteq S^a$, every vertex in $C$ is adjacent to every vertex in
    $C^a$.  Thus, there exists a maximal clique $C''$ in \skel{D} such that
    $C' \subseteq C''$.

    Consider $y \defeq \sink[D]{C''}$.  Suppose, if possible, that $x \neq y$.
    Then we must have $x \in \pa[D]{y}$ (since $x \in C''$), and also that
    $C^a \subseteq \pa[D]{y}$ (as $C^a \subseteq C''$).  Thus, we must have
    $y \in S^a$.  But then, we get that $\inb{y} \cup C \subseteq C'' \cap S^a$,
    which contradicts the assumption that $C$ is a maximal clique in $H$.  Thus,
    we must have $x = y$, so that $x$ is a \sinkv{} node in $D$.  This proves
    \cref{item:3} of the proposition.
\end{enumerate}
\end{proof}

\section{Some folklore results}
\label{sec:some-folkl-results}

We collect here some well known and folklore results concerning
$\mathcal{I}$-essential graphs that are used in our proofs.  We begin by
restating the characterization by \cite{HB12} of $\mathcal{I}$-essential graphs
(\cref{thm:hb-i-essential}).  Recall that in the main paper, this
characterization was only used in the setting of DAGs without v-structures, in
the proof of \cref{theorem:main-theorem}.  Here, we will need to use it in the
setting of general graphs. (\Cref{fig:strong-protection} in the statement of the
theorem can be found on page~\pageref{fig:strong-protection}).  Recall also that
we always assume that every intervention set contains the empty intervention,
but the empty intervention is not counted in the size of an intervention set.

\hbcharthm*

\begin{remark}\label{supp:rem:hauser-buhlman}
  Strictly speaking, Theorem 18 of \cite{HB12} only identifies the class of all
  $\mathcal{I}$-essential graphs. However, it is well known, and follows easily
  from their results that $H$ satisfies all the four items in the statement of
  \cref{thm:hb-i-essential} along with the additional conditions in the theorem
  (i.e., $H$ has the same skeleton as $D$, all directed edges of $H$ are
  directed in the same direction as in $D$, and all v-structures of $D$ are
  directed in $H$), if and only if $H = \mathcal{E}_{\mathcal{I}}(D)$.
\end{remark}
For completeness, we provide a proof of the above folklore remark in
Section~\ref{supp:sec:remark-proof}. Here, we proceed to state a couple of easy
and folklore corollaries of this characterization.  The first of these,
\cref{corollary:supp:HB14Lemma1} below, has a proof similar to the proof of
Lemma 1 of \cite{hauser_two_2014}.  For completeness, we provide this proof in
Section~\ref{supp:sec:separate-chain-corollary}.
\begin{corollary}
\label{corollary:supp:HB14Lemma1}
  Let $D$ be an arbitrary DAG, and let $\mathcal{I}$ be any intervention set
  containing the empty set.  Let $CC$ be the set of chordal chain components of
  the essential graph $\mathcal{E}(D) = \mathcal{E}_{\inb{\emptyset}}(D)$ of $D$.
  For each $S \in CC$, define
  $\mathcal{I}_S \defeq \inb{I \cap S \st I \in \mathcal{I}}$ to be the
  projection of the intervention set $\mathcal{I}$ to $S$.  Consider vertices
  $a$ and $b$ that are adjacent in $D$.  Then, the edge between $a$ and $b$ is
  directed in the $\mathcal{I}$-essential graph $\mathcal{E}_{\mathcal{I}}(D)$
  if and only if one of the following conditions is true:
  \begin{enumerate}
  \item $a$ and $b$ are elements of distinct chain components of the
    observational essential graph $\mathcal{E}(D)$ (so that the edge between $a$
    and $b$ is already directed in $\mathcal{E}(D)$).
  \item $a$ and $b$ are in the same chain components $S \in CC$ of
    $\mathcal{E}(D)$ and the edge between $a$ and $b$ is directed in the
    $\mathcal{I}_S$-essential graph $\mathcal{E}_{\mathcal{I}_s}(D[S])$ of the
    induced DAG $D[S]$.
  \end{enumerate}
  In particular, $\mathcal{E}_{\mathcal{I}}(D) = D$ if and only if
  $\mathcal{E}_{\mathcal{I}_S}(D[S]) = D[S]$ for every $S \in CC$.
\end{corollary}

The second corollary formalizes the intuitive fact that ``breaking up'' an
intervention $I$ into two smaller interventions provides at least as much
``information'' as the intervention $I$ itself.  Again, for the sake of
completeness, we provide the proof in Section~\ref{sec:proof-coroll-partition}.

\begin{corollary}
  \label{cor:partition-intervention}
  Let $D$ be an arbitrary DAG, and let $\mathcal{I}$ be any intervention set
  containing the empty set. Let $\mathcal{I}'$ be an intervention set obtained
  from $\mathcal{I}$ by ``breaking up'' a non-empty intervention in
  $\mathcal{I}$: formally,
  $\mathcal{I}' = \inp{\mathcal{I} \setminus \inb{I}} \cup \inb{I^1, I^2} $
  where $I^1$, $I^2$ are distinct non-empty sets such that $I = I^1 \cup I^2$.
  Then, any edge that is directed in $\mathcal{E}_{\mathcal{I}}(D)$ is also directed in
  $\mathcal{E}_{\mathcal{I}'}(D)$.
\end{corollary}

\section{Other Omitted Proofs}
\subsection{Proof of \cref{theorem:main-theorem-arbitrary-DAGs}}
\label{sec:supp:proof:main-arbitrary-DAGS}
\begin{proof}[Proof of \cref{theorem:main-theorem-arbitrary-DAGs}]
  \cref{corollary:supp:HB14Lemma1} says that an intervention set $\mathcal{I}$
  learns $D$ starting with $\mathcal{E}(D)$ if and only if
  $\mathcal{E}_{\mathcal{I}_S}(D[S]) = D[S]$ for every $S \in CC$. If
  $\mathcal{I}$ is a set of atomic interventions, then for each
  $I \in \mathcal{I}$, $\abs{I \cap S} = 0$ for all but one of the $S \in
  CC$. This means that if an intervention set $\mathcal{I}$ of atomic
  interventions is such that $\mathcal{E}_{\mathcal{I}}(D) = D$, then
  $\abs{\mathcal{I}} = \sum_{S \in CC} \abs{\mathcal{I_S}}$, and
  $\mathcal{E}_{\mathcal{I}_S}(D[S]) = D[S]$ for every $S \in CC$, where
  $\mathcal{I_S}$ is a set of atomic interventions defined as in
  \cref{corollary:supp:HB14Lemma1}.  By definition of $\mathcal{E}(D)$, $D[S]$
  is a DAG without v-structures for every $S \in CC$.  Thus, by
  \cref{theorem:main-theorem} we have,
  $\abs{\mathcal{I_S}} \geq \ceil{\frac{\abs{S} - r(S)}{2}}$ which implies,
  \[
    \abs{\mathcal{I}} \geq \sum_{S \in CC} \ceil{\frac{\abs{S} - r(S)}{2}} \geq
    \ceil{\sum_{S \in CC} {\frac{\abs{S} - r(S)}{2}}} = \ceil{\frac{n -
        r}{2}}.\] This completes the proof.
\end{proof}
\subsection{Proof of \cref{thm:tight-lower-bound-general}}
\label{sec:supp:general-DAGs-tight-lower-bound}
Here we restate \cref{thm:tight-lower-bound-general} and provide its proof.
\begin{theorem}[Restatement of \cref{thm:tight-lower-bound-general}]
  \label{theorem:alg-sink-nodes-arbitrary-DAGs}
  Let $D$ be an arbitrary DAG and let $\mathcal{E}(D)$ be the chain graph with
  chordal chain components representing the MEC of $D$. Let $CC$
  denote the set of chain components of $\mathcal{E}(D)$, and $r(S)$ the
  number of maximal cliques in the chain component $S \in CC$. Then, there exists
  a set $I$ of atomic interventions of size at most
  $\sum_{S \in CC}\left(\abs{S} - r(S)\right) = n - r$,
  such that $I$ fully orients $\mathcal{E}(D)$ (i.e., $\mathcal{E}_I(D) = D$),
  where $n$ is the number of nodes in $D$, and $r$ is the total number of maximal
  cliques in the chordal chain components of $\mathcal{E}(D)$ (including chain
  components consisting of singleton vertices).
\end{theorem}
\begin{proof}
  By definition of $\mathcal{E}(D)$, $D[S]$ is a DAG without v-structures for
  every $S \in CC$. \cref{theorem:correctness-alg-sink-nodes} therefore implies
  that for each $S \in CC$ there is a set $\mathcal{I}_S$ of atomic
  interventions such that $\abs{\mathcal{I}_S} \leq \abs{S} - r(S)$ and
  $\mathcal{E}_{\mathcal{I}_S}(D[S]) = D[S]$. Now, let
  $\mathcal{I} = \cup_{S \in CC} \mathcal{I}_S$.
  $\mathcal{E}_{\mathcal{I}}(D) = D$ by \cref{corollary:supp:HB14Lemma1}, and
  $\abs{\mathcal{I}} = \sum_{S \in CC}\abs{\mathcal{I}_S}$, which means
  $\abs{\mathcal{I}} \leq \sum_{S \in CC} (\abs{S} - r(S)) = n - r$.  This shows
  that there is a set of atomic interventions of size at most $n - r$ which
  fully orients $\mathcal{E}(D)$.
\end{proof}

\subsection{Proof of \cref{lemma:clique-number-bound}}
\label{sec:supp:proof-clique-number-bound}
\begin{proof}[Proof of \cref{lemma:clique-number-bound}]
  Let $C$ be a (necessarily maximal) clique of $G$ of size $\omega$.  Since $C$
  is a maximal clique of the chordal graph $G$, there exists a perfect
  elimination ordering $\sigma$ of $G$ that starts with $C$. (This is a
  consequence, e.g., of the structure of the lexicographic breadth-first-search
  algorithm used to find perfect elimination orderings of chordal graphs: see,
  e.g., the paragraph before Proposition 1 of \cite{hauser_two_2014} and
  Algorithm 6 of \cite{HB12} for a proof.  It can also be seen as a consequence
  of the maximal cardinality search algorithm of \cite{tarjan_simple_1984}: see
  Theorem 2.5 of \cite{blair_introduction_1993}.)

  Now, let $D$ be the DAG obtained by orienting the edges of $G$ according to
  $\sigma$ (i.e., the edge $u - v \in G$ is directed as $u \rightarrow v$
  in $D$ if and only if $\sigma(u) < \sigma(v)$). Suppose that
  $\sink[D]{C} = s$.  Note that $C$ cannot contain the node \sink[D]{C'} for any
  other maximal clique $C'$ since, as $\sigma$ starts with $C$, this would imply
  $C' \subseteq C$ and would contradict the maximality of $C'$.  Thus, there are
  $|\mathcal{C}| - 1$ \sinkv{} nodes of $D$ other than $s$ by
  \Cref{lemma:sink-nodes}, and, by the
  above observation, they occur in $\sigma$ after $C$. Thus,
  $n \geq \abs{C} + |\mathcal{C}| - 1$, which gives
  $n - |\mathcal{C}| \geq \omega - 1$ as $\abs{C} = \omega$.
\end{proof}

\section{Various Example Graphs}
\label{sec:supp:graph-examples}
In \Cref{subsec:comparison-with-known-lb}, we proved that our universal lower
bound is always at least as good as the previous best universal lower bound
given by \cite{squires2020}, and also gave examples of graph families where our
bound is significantly better.  We also pointed out that our lower bound and the
lower bound of \cite{squires2020} are close only in certain special
circumstances.  We now make give more details of these special cases.

We work with the same notation as that used in \cref{lemma:clique-number-bound}:
$G$ is an undirected chordal graph, $n$ is the number of nodes in $G$, $\omega$
is the size of its largest clique, and $\mathcal{C}$ is the set of its maximal
cliques.  From \cref{lemma:clique-number-bound}, it follows that for our lower
bound of $\ceil{\frac{n - \abs{\mathcal{C}}}{2}}$ and the lower bound of
$\floor{\frac{\omega}{2}} = \ceil{\frac{\omega-1}{2}}$ of \cite{squires2020} to
be equal, one of the following conditions must be true: either (i)
$n - \abs{\mathcal{C}} = \omega - 1$, or (ii) $\omega$ is even and
$n - \abs{\mathcal{C}} = \omega$.

Now consider the perfect elimination ordering $\sigma$ of $G$ used in the proof
of \cref{lemma:clique-number-bound}, and let $D$ be the DAG with skeleton $G$
constructed by orienting the edges of $G$ in accordance with $\sigma$.  Note
that by the construction of $\sigma$, the $\omega$ vertices of a largest clique
$C$ of $G$ are the first $\omega$ vertices in $\sigma$.  Note also that
$n - \abs{\mathcal{C}}$ is the number of vertices in $G$ that are \emph{not}
\sinkv{} nodes of $D$ (by \Cref{lemma:sink-nodes}).

Thus, it follows that condition (i) above for the two lower bounds to be equal
can hold only when $D$ is such that \emph{all} nodes of $G$ outside the largest
clique $C$ of $G$ are \sinkv{} nodes of $D$.  In other words, $\sigma$ is a
clique block ordering, in the sense of P1 of \cref{def:nice-ordering} of a
\gcb{} orderings, in which the first clique block $L_1(\sigma)$ consists of the
largest clique $C$, while \emph{all} other clique blocks $L_i(\sigma), i \geq 2$
are of size exactly $1$.  Similarly, condition (ii) above for the two lower
bounds to be equal can hold only when $D$ is such that \emph{all but one} of the
nodes of $G$ outside the largest clique $C$ of $G$ are \sinkv{} nodes of $G$.

We now give examples of two special families of chordal graphs where
the above conditions for the equality of the two lower bounds hold: Split graphs
and $k$-trees. Here, $\mathcal{C}(G)$ will denote the set of maximal cliques of
graph $G$.

\textit{Split graphs.} $G$ is a split graph if its vertices can be partitioned
into a clique $C$ and an independent set $Z$. For such a $G$, one of the
following possibilities must be true.
\begin{enumerate}
  \item $\exists x \in Z$ such that $C \cup \{x\}$ is complete. In this case,
  $C \cup \{x\}$ is a maximum clique and $Z$ is a maximum independent set.
  \item $\exists x \in C$ such that $Z \cup \{x\}$ is independent. In this case,
  $Z \cup \{x\}$ is a maximum independent set and $C$ is a maximum clique.
\item $C$ is a maximal clique and $Z$ is a maximal independent set. In this
  case, $C$ must also be a \emph{maximum} clique and $Z$ a \emph{maximum}
  independent set.
\end{enumerate}
For each of these cases, we have $n - \abs{\mathcal{C}(G)} = \omega - 1$, 
which implies \ceil{\frac{n - \abs{\mathcal{C}(G)}}{2}} = \floor{\frac{\omega}{2}}.

\textit{$k$-trees.} A $k$-tree is formed by starting with $K_{k+1}$ (complete
graph with $k+1$ vertices) and repeatedly adding vertices in such a way that
each added vertex $v$ has exactly $k$ neighbors, and such that these neighbors
along with $v$ form a clique.  Thus, each added vertex creates exactly one
clique of size $k+1$.  In particular, in a $k$-tree, all maximal cliques are of
size $k+1$. So, in a $k$-tree $G$ with $n = k+1+r$ nodes, we have,
$\abs{\mathcal{C}(G)} = 1+r$ and $\omega = k+1$, which implies
$n - \abs{\mathcal{C}(G)} = \omega - 1$.  Thus,
$\ceil{\frac{n - \abs{\mathcal{C}(G)}}{2}} = \ceil{\frac{\omega - 1}{2}} =
\floor{\frac{\omega}{2}}$.

In contrast to the above two families, \emph{block graphs} are an example family
of chordal graphs where our lower bound can be significantly better.
Construction 1 and 2 presented in \Cref{subsubsec:examples} are examples of
block graphs, and as discussed there, our lower bound can be $\Theta(k)$ times
the previous best universal lower bound for block graphs, where $k$ can be as
large as $\Theta(n)$ (where $n$ is the number of nodes in the graph).

\section{Details of Experimental Setup}
\label{sec:supp:experiments}
\paragraph{Experiment 1}
For this experiment, we generate $1000$ graphs from Erd\H{o}s-R\'enyi graph
model $G(n,p)$: for each of these graphs, the number of nodes $n$ is a random
integer in $[5, 25]$ and the connection probability $p$ is a random value in
$[0.1, 0.3)$.  Each of these graphs is then converted to a DAG without
v-structures, using the following procedure.  First, the edges of $G$ are
oriented according to a topological ordering $\sigma$ which is a random
permutation of the nodes of $G$: this converts $G$ into a DAG $D$ (possibly with
v-structures). Now, the nodes of $D$ are processed in a \emph{reverse} order
according to $\sigma$ (i.e., nodes coming later in $\sigma$ are processed first)
and whenever we find two non-adjacent parents, $a$ and $b$ of the current node
$u$ being processed, we add an edge $a \rightarrow b$ in $D$ if
$\sigma(a) < \sigma(b)$, and $b \rightarrow a$ in $D$ if
$\sigma(b) < \sigma(a)$.  Since nodes are processed in an order that is a
reversal of $\sigma$, this procedure ensures that the resulting DAG $D$ has no
v-structures.

In \Cref{fig:supp:opt-vs-lower-bound-seeds}, we provide plots from four further
runs of Experiment 1.  These plots use exactly the same set-up and procedure as
the plot given in \Cref{fig:opt-vs-lower-bound} on
\pageref{fig:opt-vs-lower-bound}, and differ only in the initial seed provided
to the underlying pseudo-random number generator.  These seeds are used for
generation of random graphs as well as for generating $n$ and $p$. To avoid any
post selection bias, the seeds for these plots were formed using the decimal
expansion of $\pi$ after skipping first $1015$ digits in the decimal expansion,
and then taking the next $10$ digits as the first seed, $10$ consecutive digits
after that as the second seed, and so on.  Our interpretation and inferences
from these further runs remain the same as that reported in
\Cref{sec:empir-expl} for the run underlying \Cref{fig:opt-vs-lower-bound}.

\begin{figure*}[t]
    \centering
    \renewcommand{\thesubfigure}{\roman{subfigure}}
    \begin{subfigure}[b]{0.45\textwidth}
      \centering
      \includegraphics[scale=0.45]{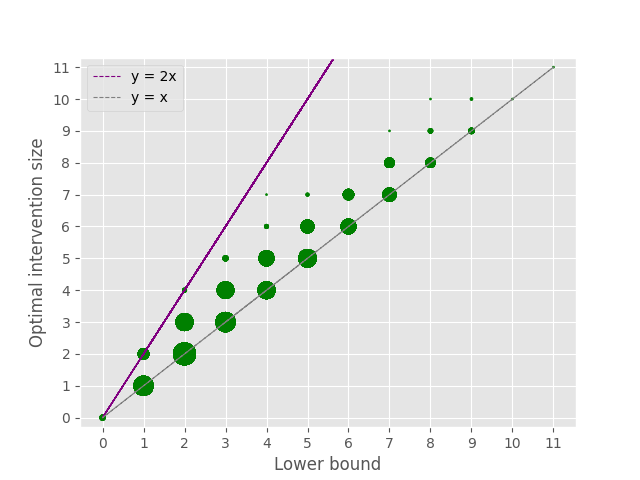}
      \caption{Run $1$}
      \label{fig:supp:exp1-seed1}
    \end{subfigure}
    \begin{subfigure}[b]{0.45\textwidth}
      \centering
      \includegraphics[scale=0.45]{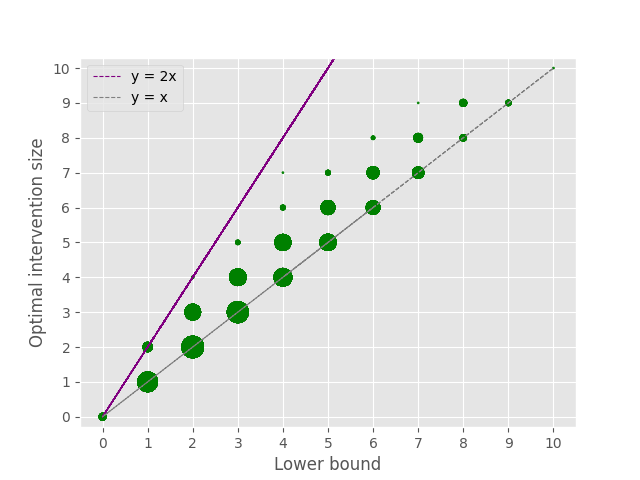}
      \caption{Run $2$}
      \label{fig:supp:exp1-seed2}
    \end{subfigure}
    \begin{subfigure}[b]{0.45\textwidth}
      \centering
      \includegraphics[scale=0.45]{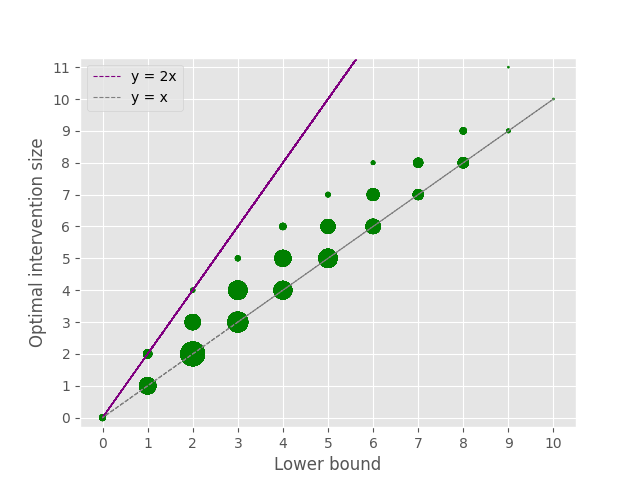}
      \caption{Run $3$}
      \label{fig:supp:exp1-seed3}
    \end{subfigure}
    \begin{subfigure}[b]{0.45\textwidth}
      \centering
      \includegraphics[scale=0.45]{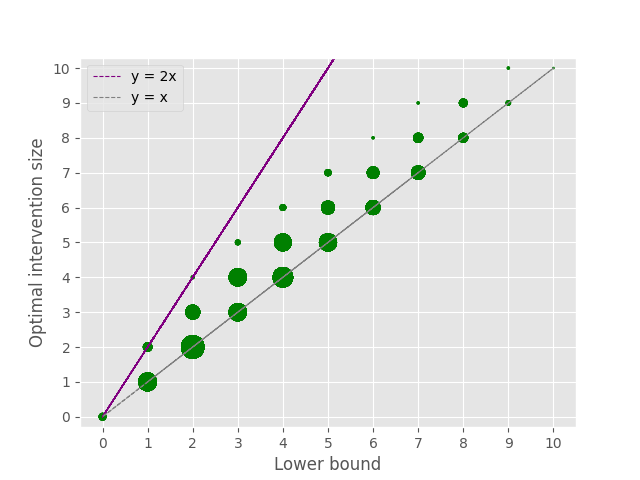}
      \caption{Run $4$}
      \label{fig:supp:exp1-seed4}
    \end{subfigure}
    \vspace{\baselineskip}
  \caption{Experiment 1 Runs with Varying Seeds}
  \label{fig:supp:opt-vs-lower-bound-seeds}
\end{figure*}

\paragraph{Experiment 2}
For this experiment, we generate $1000$ random DAGs without v-structures for
each size in $\inb{10, 20, 30, 40, 50, 60}$.  We now describe the procedure for
generating a DAG $D$ (without v-structures) with $n$ nodes, other than $n$, this
procedure takes two more inputs, ${min\_clique\_size}$ and
${max\_clique\_size}$. If ${min\_clique\_size} = X$ and $max\_clique\_size = Y$,
we try to keep the size of all cliques of $D$ in $[X, Y]$. First, we initialize
a DAG $D$ with nodes $0, \dots, n-1$, and no edges. We take
$\sigma = (0, \dots, n-1)$ to be a perfect elimination ordering of $D$. We then
process the nodes of $D$ in reverse order of $\sigma$.  When node $u$ is being
processed, we first compute the number of parents that $u$ already has in $D$.
Now we compute lower and upper bounds $\ell_1 \geq 0$ and $\ell_2 \geq 0$ on the
number of parents that could be added to the set of parents of $u$ while still
keeping the total number of parents below $Y - 1$, and at least $X - 1$.  (Note
that $\ell_1, \ell_2 \leq \abs{\inb{0, \dots, u - 1} \setminus \pa[D]{u}}$,
since the latter is the number of currently available vertices that could be
added to the parent set of $u$).  We now choose an integer $\ell$ uniformly at
random from $[\ell_1, \ell_2]$: this will be the number of new parents to be
added to the set of parents of $u$.  Note that it may happen that $\ell = 0$,
for example when $u$ already has $Y - 1$ or more parents, so that $\ell_2 = 0$.
Next, we sample a set $Z$ of $\ell$ nodes (without replacement) from
$(\inb{0, \dots, u - 1} \setminus \pa[D]{u})$, and add the edges
$z \rightarrow u$ to $D$ for each $z \in Z$.  Further, for any two non-adjacent
parents of $u$, we add an edge $a \rightarrow b$ to $D$ if
$\sigma(a) < \sigma(b)$, and $b \rightarrow a$ if $\sigma(b) < \sigma(a)$.  This
makes sure that there are no v-structures in $D$. Note that, as described in
\Cref{sec:empir-expl}, the procedure used here only tries to keep the maximum
clique size bounded above by $Y$, but it can overshoot and produce a graph with
a clique of size larger than $Y$ as well. In our experiments, we take
$min\_clique\_size = 2$, $max\_clique\_size = 4$.

In \Cref{fig:supp:comparison-lower-bounds-seeds}, we provide plots from four
further runs of Experiment 2.  These plots use exactly the same set-up and
procedure as the plot given in \Cref{fig:comparison-lower-bounds} in
\Cref{sec:empir-expl}, and differ only in the initial seed provided to the
underlying pseudo-random number generator.  Again, to avoid post-selection bias, we use seeds given by the procedure given
for Experiment 1 above.  Our interpretation and inferences from these further
runs remain the same as that reported in \Cref{sec:empir-expl} for the run
underlying \Cref{fig:comparison-lower-bounds}.

\begin{figure*}[t]
    \centering
    \renewcommand{\thesubfigure}{\roman{subfigure}}
    \begin{subfigure}[b]{0.45\textwidth}
      \centering
      \includegraphics[scale=0.45]{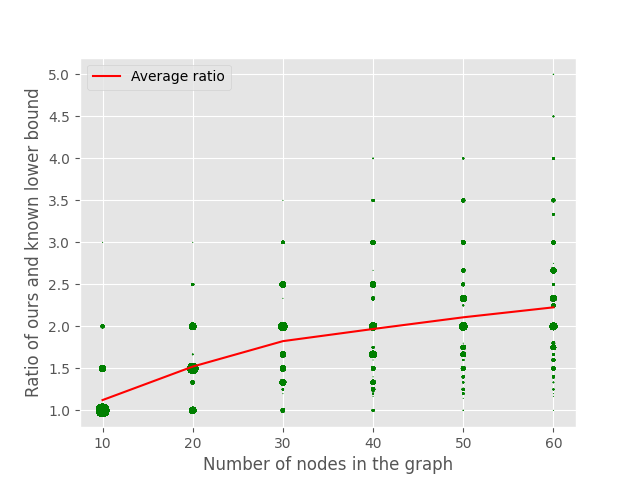}
      \caption{Run $1$}
      \label{fig:supp:exp2-seed1}
    \end{subfigure}
    \begin{subfigure}[b]{0.45\textwidth}
      \centering
      \includegraphics[scale=0.45]{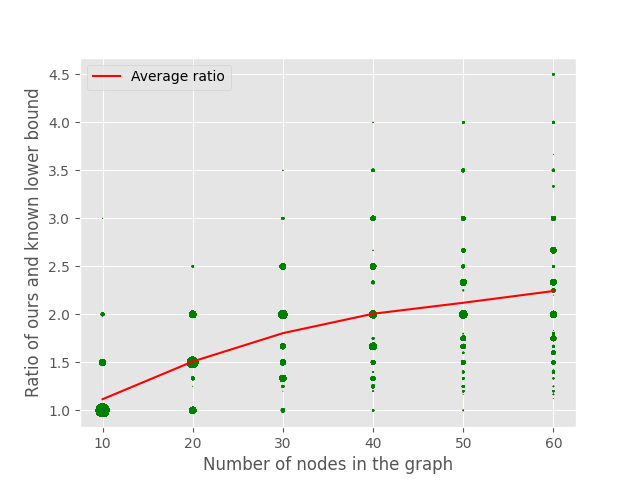}
      \caption{Run $2$}
      \label{fig:supp:exp2-seed2}
    \end{subfigure}
    \begin{subfigure}[b]{0.45\textwidth}
      \centering
      \includegraphics[scale=0.45]{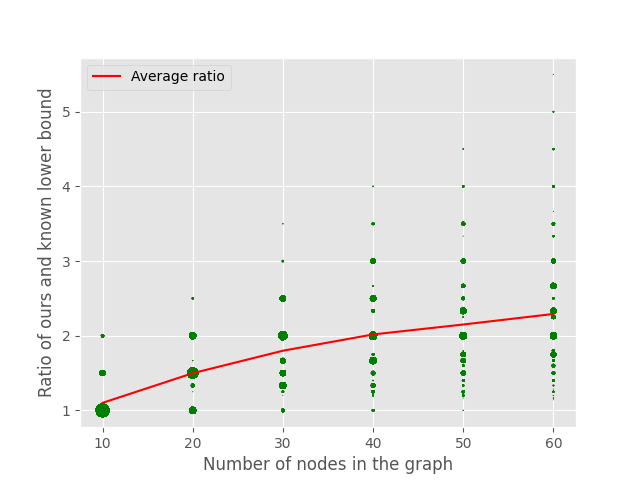}
      \caption{Run $3$}
      \label{fig:supp:exp2-seed3}
    \end{subfigure}
    \begin{subfigure}[b]{0.45\textwidth}
      \centering
      \includegraphics[scale=0.45]{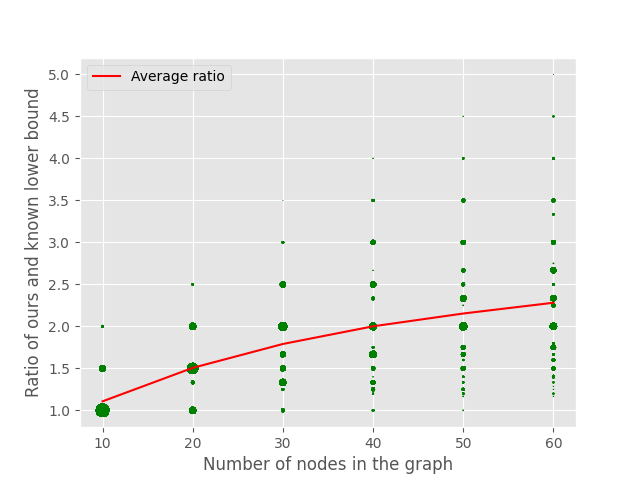}
      \caption{Run $4$}
      \label{fig:supp:exp2-seed4}
    \end{subfigure}
    \vspace{\baselineskip}
  \caption{Experiment 2 Runs with Varying Seeds}
  \label{fig:supp:comparison-lower-bounds-seeds}
\end{figure*}

\section{Omitted proofs of folklore results}
In this section, we provide the proofs of
\cref{supp:rem:hauser-buhlman,cor:partition-intervention,corollary:supp:HB14Lemma1}.
As described above, we believe these results to be folklore and well known, and
provide the proofs below only for the sake of completeness.

\subsection{Proof of Remark~\ref{supp:rem:hauser-buhlman}}
\label{supp:sec:remark-proof}
  \begin{proof}[Proof of \cref{supp:rem:hauser-buhlman}]
    By Theorem 10(iv) of \cite{HB12}, it follows that
  $\mathcal{E}_{\mathcal{I}}(D)$ must have the same skeleton and the same
  v-structures as $D$, and must also have all its directed edges directed in the
  same direction as in $D$.  This proves that $\mathcal{E}_{\mathcal{I}}(D)$
  satisfies all the conditions of \cref{thm:hb-i-essential}.  To complete the
  proof, we now show that it is the only graph satisfying all the conditions of the theorem.

  For if not, then let $G$ and $H$ be two different graphs satisfying all
  the conditions of \cref{thm:hb-i-essential}.  Thus, $G$ and $H$ have the same
  skeleton as $D$, all their directed edges are in the same direction as in $D$,
  and further, all v-structures of $D$ are directed in both $G$ and $H$.  If
  $G \neq H$, the set $E'$ of edges that are directed in $G$ but not in $H$ is
  therefore non-empty (possibly after interchanging the labels $G$ and $H$).
  Fix a topological ordering $\sigma$ of $D$, and let $a \rightarrow b \in E'$
  be such that $\sigma(b) \leq \sigma(b')$ for all $a' \rightarrow b' \in
  E'$. Since $a \rightarrow b$ in $G$, $a \rightarrow b$ must be strongly
  $\mathcal{I}$-protected in $G$. Now, there cannot exist $J \in \mathcal{I}$
  such that $|J \cap \inb{a, b}| = 1$, since in that case $a \rightarrow b$
  would be directed in $H$ as well (by \cref{item:3} of
  \cref{thm:hb-i-essential}).  Thus, at least one of the four graphs in
  \Cref{fig:strong-protection} must appear as an induced subgraph of $G$, with
  $a \rightarrow b$ appearing in that induced subgraph in the configuration
  indicated in the figure. If subgraph (i) appears as an induced subgraph of
  $G$, then we must have $c \rightarrow a$ in $H$ since $\sigma(a) < \sigma(b)$,
  but this means that $c \rightarrow a - b$ would be an induced subgraph of $H$,
  contradicting \cref{item:directed-by-Meek-rule-1} of
  \cref{thm:hb-i-essential}.  Similarly, $a \rightarrow b$ cannot be in the
  configuration indicated in subgraph (ii), since any v-structure in $G$ is
  directed in $H$, so that $a \rightarrow b$ would be directed in $H$ as
  well. If subgraph (iii) appears as an induced subgraph of $G$,
  $a \rightarrow c$ must be directed in $H$, as $\sigma(c) < \sigma(b)$, but
  this would mean that $H$ contains a directed cycle $a, c, b, a$ (since $a - b$
  is undirected in $H$), and this contradicts that $H$ is a chain
  graph.\footnote{Recall that a directed cycle in a general graph is a cycle in
    which all directed edges point in the same direction, and in which at least
    one edge is directed.  The formal definition is given in
    \Cref{sec:preliminaries}.} If subgraph (iv) appears as an induced subgraph
  of $G$, then $c_1 \rightarrow b \leftarrow c_2$ appears in $H$ as well since
  any v-structure of $G$ must also be directed in $H$.  Further, at least one of
  the following four configurations must appear in $H$: (a) $a \rightarrow c_1$
  (b) $a \rightarrow c_2$ (c) $a - c_1$ (d) $a - c_2$ (for if not, then
  $c_1 \rightarrow a \leftarrow c_2$ would be a v-structure in $H$ that is not
  directed in $G$, contradicting that all v-structures of $D$ are directed in
  both $G$ and $H$).  However, if any of the four configuration appears in $H$,
  we get a directed cycle in $H$ (since $a - b$ is undirected in $H$), which
  contradicts the fact that $H$ is a chain graph.

  We conclude therefore that $E'$ must in fact be empty and hence $G = H$.
  Thus, given a DAG $D$, the unique graph satisfying all conditions of
  \cref{thm:hb-i-essential} is the $\mathcal{I}$-essential graph
  $\mathcal{E}_{\mathcal{I}}(D)$ of $D$.
\end{proof}

\subsection{Proof of Corollary~\ref{corollary:supp:HB14Lemma1}}
\label{supp:sec:separate-chain-corollary}
\begin{proof}[Proof of~\cref{corollary:supp:HB14Lemma1}]
  Let $H$ be the graph with the same skeleton as $D$ in which exactly the edges
  satisfying one of the two conditions of the corollary are directed.  We prove
  that $H$ satisfies all the conditions of \Cref{thm:hb-i-essential}, and must
  therefore be the same as $\mathcal{E}_{\mathcal{I}}(D)$ (see also
  \cref{supp:rem:hauser-buhlman}).  This will complete the proof of the
  corollary.

  Recall that by construction, any edge of $H$
  is directed if and only if
  \begin{enumerate}
  \item the endpoints of the edge are in different chain components of
    $\mathcal{E}(D)$, so that it is already directed in $\mathcal{E}(D)$, or
  \item the endpoints of the edge lie in the same chain component $S$ of
    $\mathcal{E}(D)$, and the edge is directed in
    $\mathcal{E}_{\mathcal{I}_S}(D_S)$.
  \end{enumerate}
  In particular, item 1 implies that any edge that is directed in
  $\mathcal{E}(D)$ is also directed in $H$ (since all directed edges of a chain
  graph have their endpoints in different chain components).

  We now verify that $H$ satisfies all the conditions of
  \cref{thm:hb-i-essential}. By construction, $H$ has the same skeleton as $D$,
  and all its directed edges are directed in the same direction as $D$.
  Further, all the v-structures of $D$ are directed in $H$, since these are
  directed in $\mathcal{E}(D)$.

  Any directed cycle $C$ in $H$ would be a directed cycle either in
  $\mathcal{E}(D)$ (in case $C$ includes vertices from at least two different
  chain components of $\mathcal{E}(D)$), or in
  $\mathcal{E}_{\mathcal{I}_S}(D_S)$ for some chain component $S$ of
  $\mathcal{E}(D)$ (in case $C$ is contained within a single chain component $S$
  of $\mathcal{E}(D)$).  Since both $\mathcal{E}(D)$ and
  $\mathcal{E}_{\mathcal{I}_S}(D)$ are chain graphs (from
  \cref{thm:hb-i-essential}), they do not have any directed cycles.  It
  therefore follows that $H$ cannot have a directed cycle either, and hence is
  a chain graph.  Further, the chain components of $H$ are induced subgraphs of
  the chain components of $\mathcal{E}(D)$.  Since the chain components of
  $\mathcal{E}(D)$ are chordal (again from \cref{thm:hb-i-essential}), it
  follows that the chain components of $H$ are also chordal.  Thus, $H$
  satisfies \cref{item:chain-chordal} of \cref{thm:hb-i-essential}.

  Suppose now that, if possible, $H$ has an induced subgraph of the form
  $a \rightarrow b - c$. Thus, the edge $b - c$ must be undirected in
  $\mathcal{E}(D)$ as well, so that $b$ and $c$ are in the same chain component
  $S$ of $\mathcal{E}(D)$.  If $a$ is also in $S$, then $a \rightarrow b - c$
  would be an induced sub-graph of the interventional essential graph
  $\mathcal{E}_{\mathcal{I}_S}(D_S)$, which would contradict
  \cref{item:directed-by-Meek-rule-1} of \cref{thm:hb-i-essential}.  Similarly,
  if $a$ is not in $S$, then $a \rightarrow b$ would be directed in
  $\mathcal{E}(D)$, so that $a \rightarrow b - c$ would be an induced sub-graph
  of the essential graph $\mathcal{E}(D) = \mathcal{E}_{\inb{\emptyset}}(D)$,
  again contradicting \cref{item:directed-by-Meek-rule-1} of
  \cref{thm:hb-i-essential}.  We conclude that an induced subgraph of the form
  $a \rightarrow b - c$ cannot occur in $H$. Thus, $H$ satisfies
  \cref{item:directed-by-Meek-rule-1} of \cref{thm:hb-i-essential}.

  To verify \cref{item:directed-by-intervention}, consider any two adjacent
  vertices $a$ and $b$ in $H$ such that $\abs{I \cap \inb{a, b}} = 1$ for some
  $I \in \mathcal{I}$.  If $a$ and $b$ are in different chain components of
  $\mathcal{E}(D)$, then the edge between them is directed in $\mathcal{E}(D)$
  and hence also in $H$.  On the other hand, if $a$ and $b$ are in the same
  chain component $S$ of $\mathcal{E}(D)$, then we have
  $\abs{(I \cap S) \cap \inb{a, b}} = \abs{I \cap \inb{a, b}} = 1$ for
  $I \cap S \in \mathcal{I}_S$, so that the edge between $a$ and $b$ is directed
  in $\mathcal{E}_{\mathcal{I}_S}(D_S)$ (by \cref{item:directed-by-intervention}
  of \cref{thm:hb-i-essential}) and hence also in $H$.  It thus follows that $H$
  satisfies \cref{item:directed-by-intervention} of \cref{thm:hb-i-essential}.

  Finally, we show that any directed edge in $H$ is $\mathcal{I}$-strongly
  protected.  Consider first a directed edge $a \rightarrow b$ in $H$ where $a$
  and $b$ belong to the same chain component $S$ of $\mathcal{E}(D)$.  Then,
  since $a \rightarrow b$ is directed also in
  $\mathcal{E}_{\mathcal{I}_S}(D_S)$, it must be $\mathcal{I}_S$-strongly
  protected in $\mathcal{E}_{\mathcal{I}_S}(D_S)$.  It follows directly from the
  definition of interventional strong protection and the construction of $H$
  then that $a \rightarrow b$ is $\mathcal{I}$-strongly protected in $H$ (since
  any of the configurations of \Cref{fig:strong-protection}, if present as an
  induced subgraph of $\mathcal{E}_{\mathcal{I}_S}(D_S)$, is also present as an
  induced subgraph in $H$).

  Consider now a directed edge $a \rightarrow b$ in $H$ when $a$ and $b$ are in
  different chain components of $\mathcal{E}(D)$.  Then $a \rightarrow b$ must
  be directed, and hence also $\inb{\emptyset}$-strongly protected, in
  $\mathcal{E}(D)$.  Now, if $a \rightarrow b$ appears as part of an induced
  subgraph of $\mathcal{E}(D)$ of the forms (i), (ii) or (iii) of
  \Cref{fig:strong-protection}, then the same configurarion also appears as an
  induced subgraph of $H$ (since all directed edges of $\mathcal{E}(D)$ are
  directed in $H$), so that $a \rightarrow b$ is also $\mathcal{I}$-strongly
  protected in $H$.  Suppose then that $a \rightarrow b$ appears as part of an
  induced subgraph of the form (iv) of \Cref{fig:strong-protection}.  Then, the
  vertices $a$, $c_1$ and $c_2$ appearing in the configuration must be in the
  same chain component $S$ of $\mathcal{E}(D)$ (since they are in a connected
  component of undirected edges).  It follows that the configurations
  $c_1 \rightarrow a - c_2$ and $c_2 \rightarrow a - c_1$ cannot appear in $H$.
  For, if they did, then they would also appear in the $\mathcal{I}_S$ essential
  graph $\mathcal{E}_{\mathcal{I}_S}(D_S)$, thereby contradicting
  \cref{item:directed-by-Meek-rule-1} of \cref{thm:hb-i-essential} (when applied
  to the interventional essential graph $\mathcal{E}_{\mathcal{I}_S}(D_S)$).
  The configuration $c_1 \rightarrow a \leftarrow c_2$ also cannot occur in $H$,
  since otherwise, the v-structure $c_1 \rightarrow a \leftarrow c_2$ of $D$
  could not have remained undirected in $\mathcal{E}(D)$.  It follows that at
  least one of the following three configurations must appear in $H$:
  $a \rightarrow c_1$, $a \rightarrow c_2$ or $c_1 - a - c_2$.  In the last
  case, $a \rightarrow b$ is $\mathcal{I}$-strongly protected in $H$ as
  configuration (iv) of \Cref{fig:strong-protection} appears as an induced
  subgraph of $H$ (exactly as it does in $\mathcal{E}(D)$).  In the first two
  cases, $a \rightarrow b$ is $\mathcal{I}$-strongly protected in $H$ as
  configurations (iii) of \Cref{fig:strong-protection} appears as an induced
  subgraph of $H$ (with the role of the vertex $c$ in that configuration played
  by either $c_1$ or $c_2$, as the edges $c_1 \dir b$ and $c_2 \dir b$ are both
  directed in $H$, since they are directed in $\mathcal{E}(D)$).  Thus, we see
  that every directed edge in $H$ is $\mathcal{I}$-strongly protected in $H$,
  and hence $H$ satisfies \cref{item:strong-i-protection} of
  \cref{thm:hb-i-essential} as well.

  It follows from \cref{thm:hb-i-essential} therefore that
  $H = \mathcal{E}_{\mathcal{I}}(D)$. As discussed at the beginning of the
  proof, this completes the proof of the corollary.
\end{proof}

\subsection{Proof of Corollary~\ref{cor:partition-intervention}}
\label{sec:proof-coroll-partition}
\begin{proof}[Proof of \cref{cor:partition-intervention}]
  Suppose that there exists an edge $a \rightarrow b$ which is directed in
  $\mathcal{E}_{\mathcal{I}}(D)$ but undirected in
  $\mathcal{E}_{\mathcal{I}'}(D)$.  Fix a topological ordering $\sigma$ of $D$,
  and among all such edges, choose any one with the smallest possible value of
  $\sigma(b)$.  Thus, we have $a \rightarrow b \in \mathcal{E}_{\mathcal{I}}(D)$
  and $a \undir b \in \mathcal{E}_{\mathcal{I}'}(D)$.  Also, by the choice of
  $b$, if $c \rightarrow d$ is directed in $\mathcal{E}_{\mathcal{I}}(D)$, and
  $\sigma(d) < \sigma(b)$, then $c \rightarrow d$ is also directed in
  $\mathcal{E}_{\mathcal{I}'}(D)$.

  Now, we will derive a contradiction to the requirement in
  \cref{item:strong-i-protection} of \cref{thm:hb-i-essential}, which says that
  $a \rightarrow b$ must be strongly $\mathcal{I}$-protected in
  $\mathcal{E}_{\mathcal{I}}(D)$.  We consider all the five possible ways in
  which $a \rightarrow b$ may be strongly $\mathcal{I}$-protected in
  $\mathcal{E}_{\mathcal{I}}(D)$, and derive a contradiction in each case.

  \begin{description}
  \item[Case 1] There exists $J \in \mathcal{I}$ such that
    $\abs{J \cap \inb{a, b}} = 1$.  In this case there exists
    $J' \in \mathcal{I}'$ such that $\abs{J' \cap \inb{a, b}} = 1$: in case
    $J \neq I$ we can take $J' = J$, while if $J = I$, we can take $J'$ to be
    one of $I^1$ and $I^2$, since $I = I^1 \cup I^2$.  But then, by
    \cref{item:directed-by-intervention} of \cref{thm:hb-i-essential},
    $a \rightarrow b$ must be directed in $\mathcal{E}_{\mathcal{I}'}(D)$ as
    well, which is a contradiction.
  \item [Case 2] The subgraph (i) (of the form $c \rightarrow a \rightarrow b$)
    in \Cref{fig:strong-protection} appears as an induced subgraph of
    $\mathcal{E}_{\mathcal{I}}(D)$.  In this case, $c \rightarrow a$ must be
    directed in $\mathcal{E}_{\mathcal{I}'}(D)$, since $\sigma(a) < \sigma(b)$.
    But then, the induced subgraph $c \rightarrow a \undir b$ of
    $\mathcal{E}_{\mathcal{I}'}(D)$ contradicts
    \cref{item:directed-by-Meek-rule-1} of \cref{thm:hb-i-essential}.
  \item [Case 3] The subgraph (ii) (of the form $c \rightarrow b \leftarrow a$)
    an \Cref{fig:strong-protection} appears as in induced subgraph of
    $\mathcal{E}_{\mathcal{I}}(D)$.  In this case,
    $c \rightarrow b \leftarrow a$ is a v-structure in
    $\mathcal{E}_{\mathcal{I}}(D)$, and hence also (by
    \cref{thm:hb-i-essential}) in $D$.  Thus, again by
    \cref{thm:hb-i-essential}, it must also be directed in
    $\mathcal{E}_{\mathcal{I'}}(D)$.  This contradicts the assumption that
    $a \undir b$ is undirected in $\mathcal{E}_{\mathcal{I}}(D)$.
  \item [Case 4] The subgraph (iii) in \Cref{fig:strong-protection} appears as
    an induced subgraph of $\mathcal{E}_{\mathcal{I}}(D)$.  In this case,
    $a \rightarrow c$ is directed in $\mathcal{E}_{\mathcal{I}'}(D)$, since
    $\sigma(c) < \sigma(b)$ (because of the presence of the edge
    $c \rightarrow b$ in $\mathcal{E}_{\mathcal{I}}(D)$, and hence also in $D$).
    Now, irrespective of whether the edge between $c$ and $b$ is directed or not
    in $\mathcal{E}_{\mathcal{I}'}(D)$, we have a directed cycle in
    $\mathcal{E}_{\mathcal{I}'}(D)$: this directed cycle is
    $a \rightarrow c \undir b \undir a$ in case $c \undir b$ is undirected in
    $\mathcal{E}_{\mathcal{I}}(D)$, and $a \rightarrow c \rightarrow b \undir a$
    in case $c \rightarrow b$ is directed in $\mathcal{E}_{\mathcal{I}}(D)$
    (note that \cref{thm:hb-i-essential} implies that since
    $c \rightarrow b \in \mathcal{E}_{\mathcal{I}}(D)$, it must also be present
    in $D$, so that $b \rightarrow c$ cannot be present in
    $\mathcal{E}_{\mathcal{I}'}(D)$).  However, this is a contradiction, since,
    by \cref{thm:hb-i-essential}, $\mathcal{E}_{\mathcal{I}'}(D)$ must be a
    chain graph, and hence cannot contain any directed cycles.

  \item [Case 5] The subgraph (iv) in \Cref{fig:strong-protection} appears as an
    induced subgraph of $\mathcal{E}_{\mathcal{I}}(D)$.  In this case,
    $c_1 \rightarrow b \leftarrow c_2$ is a v-structure in
    $\mathcal{E}_{\mathcal{I}}(D)$, and hence, by \cref{thm:hb-i-essential},
    must be directed in $\mathcal{E}_{\mathcal{I}'}(D)$ as well.  Now, if any
    one of the edges $a \undir c_1$ or $a \dir c_1$ or $a \undir c_2$ or
    $a \rightarrow c_2$ are present in $\mathcal{E}_{\mathcal{I}'}(D)$, we get a
    contradiction to the fact that $\mathcal{E}_{\mathcal{I}'}(D)$ must be a
    chain graph (exactly as in Case 4 above).  The only remaining possibility is
    that $c_1 \rightarrow a$ and $c_2 \rightarrow a$ are both present in
    $\mathcal{E}_{\mathcal{I}'}(D)$.  But this is a contradiction to the fact
    that the edges $c_1 \undir a $ and $c_2 \undir a$ are undirected in
    $\mathcal{E}_{\mathcal{I}}(D)$.  To see this, note that since
    $c_1 \rightarrow a \leftarrow c_2$ is a v-structure in
    $\mathcal{E}_{\mathcal{I}'}(D)$, it must (by \cref{thm:hb-i-essential}) also
    be directed in $\mathcal{E}_{\mathcal{I}}(D)$.
  \end{description}
  Thus, we conclude that every edge that is directed in
  $\mathcal{E}_{\mathcal{I}}(D)$ must also be directed in
  $\mathcal{E}_{\mathcal{I}'}(D)$.
\end{proof}

\vskip 0.2in
\bibliography{causal}

\end{document}